\documentclass[10pt,twocolumn,letterpaper]{article}

\usepackage[pagenumbers]{cvpr} 

\usepackage{graphicx}
\usepackage{amsmath}
\usepackage{amssymb}
\usepackage{booktabs}

\usepackage{amsthm,paralist}
\usepackage{graphicx}
\usepackage{relsize,comment}
\usepackage{multirow}
\usepackage{subcaption}
\usepackage{algorithm}
\usepackage{algorithmic}
\usepackage{booktabs}
\usepackage{amsmath}
\usepackage{amssymb}
\usepackage{epsfig}
\usepackage{bbm}
\usepackage{amsmath,amssymb,bm}
\usepackage{relsize}
\usepackage{graphicx}
\usepackage[inline]{enumitem}
\usepackage{array}
\usepackage{soul}
\usepackage{subcaption}
\usepackage{booktabs}
\usepackage{booktabs,siunitx,caption}
\usepackage{mathrsfs}
\usepackage[table]{xcolor}
\usepackage{pifont}
\usepackage{relsize}
\makeatletter
\@namedef{ver@everyshi.sty}{}
\makeatother
\usepackage{tikz}
\usetikzlibrary{fit,calc}
\usepackage{listings} 

\newtheorem{theorem}{Theorem}
\newtheorem{lemma}[theorem]{Lemma}
\newtheorem{definition}[theorem]{Definition}
\newtheorem{remark}[theorem]{Remark}
\newtheorem{fact}[theorem]{Fact}

\newtheorem*{theorem*}{Lemma}

\newcommand{\E}{\mathbb{E}}

\newcolumntype{P}[1]{>{\centering\arraybackslash}p{#1}}
\newcolumntype{g}{>{\columncolor[rgb]{0.89, 0.89, 1}}c}

\newcommand*{\tikzmk}[1]{\tikz[remember picture,overlay,] \node (#1) {};\ignorespaces}
\newcommand{\boxit}[1]{\tikz[remember picture,overlay]{\node[yshift=1pt,fill=#1,opacity=.12,fit={(A)($(B)+(\linewidth,.8\baselineskip)$)}] {};}\ignorespaces}
\colorlet{pink}{red!40}
\colorlet{blue}{cyan!60}
\colorlet{yellow}{yellow!50}

\newcommand{\newsym}[1]{\mbox{\footnotesize $\mathcal{#1}$}}
\newcommand{\cmark}{\ding{51}}%
\newcommand{\xmark}{\ding{55}}%
\newcommand{\lat}{\bm{\ell}}

\newtheorem{exmp}{Example}[section]

\usepackage[pagebackref,breaklinks,colorlinks]{hyperref}

\usepackage[capitalize]{cleveref}
\crefname{section}{Sec.}{Secs.}
\Crefname{section}{Section}{Sections}
\Crefname{table}{Table}{Tables}
\crefname{table}{Tab.}{Tabs.}

\def\cvprPaperID{00284} 
\def\confName{CVPR}
\def\confYear{2022}

\begin{document}

\title{Equivariance Allows Handling Multiple Nuisance Variables\\ When Analyzing Pooled Neuroimaging Datasets}


\author{{\normalsize \textbf{Vishnu Suresh Lokhande }}\\
	{\tt\footnotesize lokhande@cs.wisc.edu}
	\and
	{\normalsize \textbf{Rudrasis Chakraborty}}\\
	{\tt\footnotesize rudrasischa@gmail.com}
	\and
	{\normalsize \textbf{Sathya N. Ravi}} \\
	{\tt\footnotesize sathya@uic.edu}
	\and
	{\normalsize \textbf{Vikas Singh}} \\
	{\tt\footnotesize vsingh@biostat.wisc.edu}
	\and
}

\maketitle

\begin{abstract}
Pooling multiple neuroimaging datasets across institutions often enables improvements in statistical power when evaluating associations (e.g., between risk factors and disease outcomes) that may otherwise be too weak to detect. When there is only a {\em single} source of variability (e.g., different scanners), domain adaptation and matching the distributions of representations may suffice in many scenarios. But in the presence of {\em more than one} nuisance variable which concurrently influence the measurements, pooling datasets poses unique challenges, e.g., variations in the data can come from both the acquisition method as well as the demographics of participants (gender, age). Invariant representation learning, by itself, is ill-suited to fully model the data generation process. In this paper, we show how bringing recent results on equivariant representation learning (for studying symmetries in neural networks) instantiated on structured spaces together with simple use of classical results on causal inference provides an effective practical solution. In particular, we demonstrate how our model allows dealing with more than one nuisance variable under some assumptions and can enable analysis of pooled scientific datasets in scenarios that would otherwise entail removing a large portion of the samples. Our code is available on {\footnotesize \url{https://github.com/vsingh-group/DatasetPooling}}
\end{abstract}

\section{Introduction}
Observational studies in 
many disciplines
acquire cross-sectional/longitudinal 
clinical and imaging data 
to understand diseases such as neurodegeneration and  
dementia \cite{soldan2019atn}. Typically, these studies are 
sufficiently powered for the 
primary scientific 
hypotheses of interest. However, 
{\em secondary} analyses to investigate weaker but potentially 
interesting associations 
between risk factors (such as genetics) and disease 
outcomes 
are often difficult 
when using common 
statistical significance thresholds,  
due to the small/medium sample sizes. 

\label{sec:teaser_fig}
\begin{figure}[!t]
	\includegraphics[width=\columnwidth]{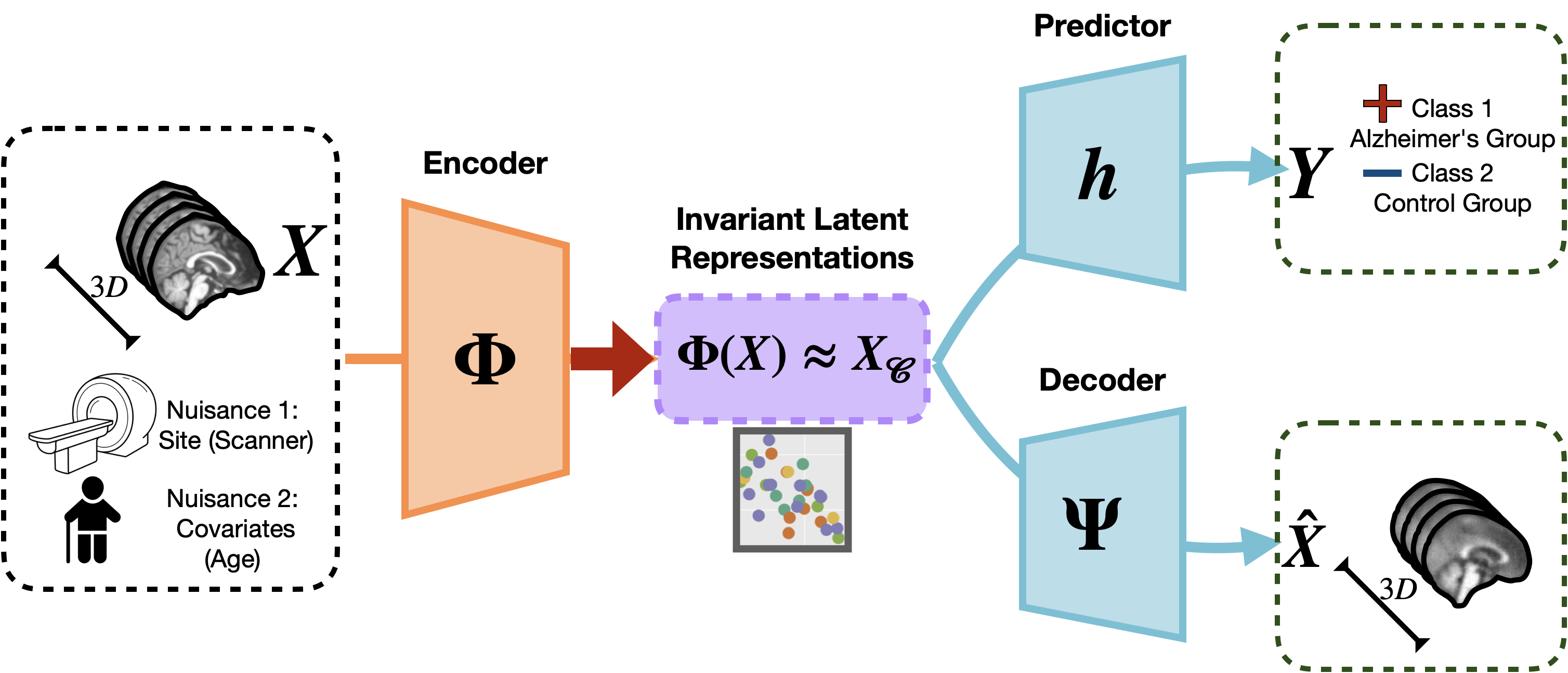}
	\centering
	\caption{\footnotesize \label{fig:encoder}  \textbf{Learning Invariant Representations.} In our framework, input images $X$ are pooled together from multiple sites. An encoder $\Phi$  maps $X$ to the latent representations $\Phi(X)$ that  corresponds to high-level causal features $X_\mathscr{C}$ that influences the label prediction. Unlike the input images $X$, $\Phi(X)$ is robust to nuisance attributes like site (scanner) and covariates (age). $\Phi$ is trained alongside predictor $h$ and decoder $\Psi$.}
	\vspace{-0.25in}
\end{figure}

\begin{figure*}[!t]
	\begin{subfigure}[b]{0.2\linewidth}
		\includegraphics[width=\linewidth]{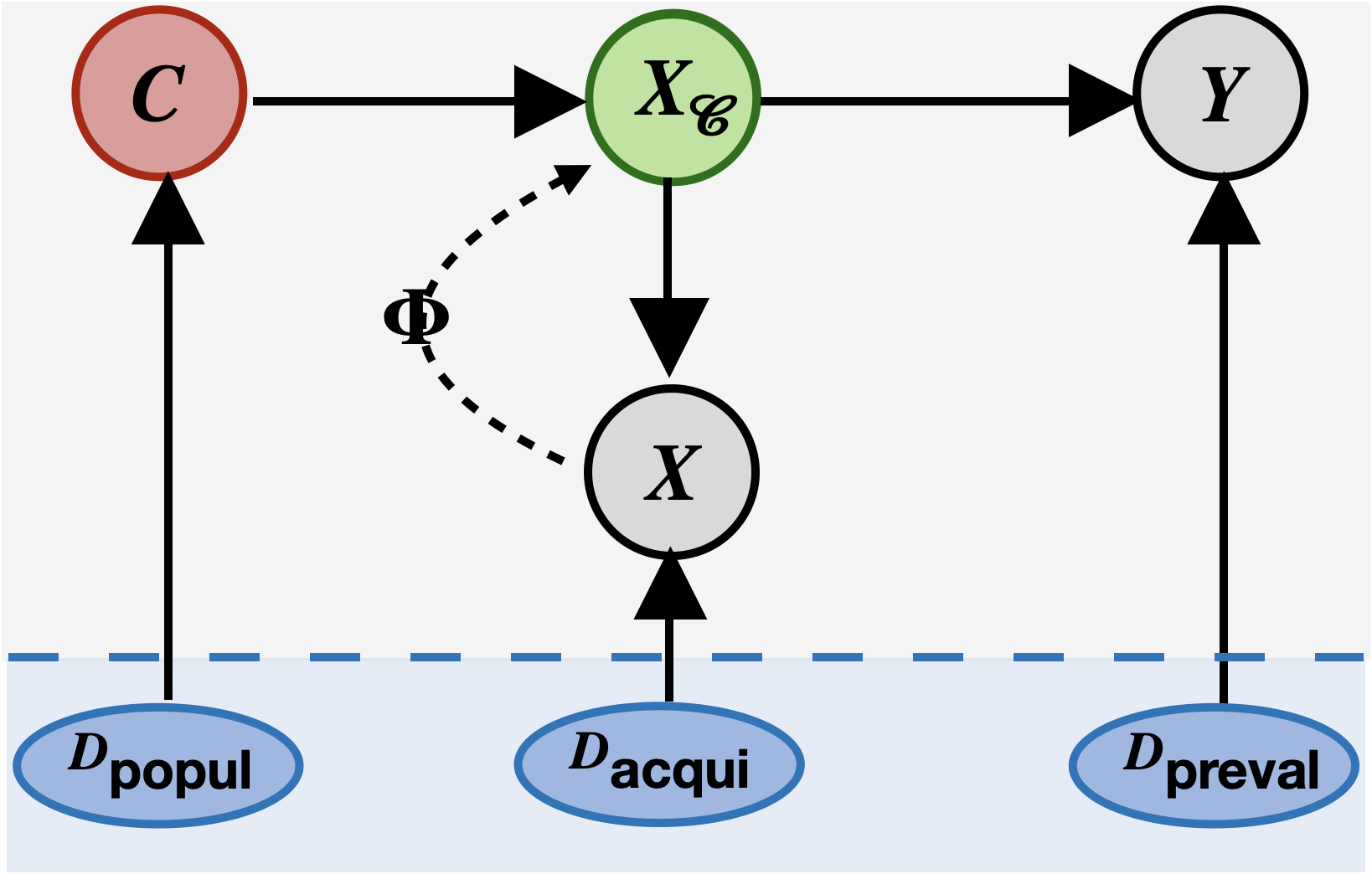}
		\caption{Causal Diagram}
		\label{fig:scm}
	\end{subfigure} \hspace{4mm} %
	\begin{subfigure}[b]{0.37\linewidth}
		\includegraphics[width=\linewidth]{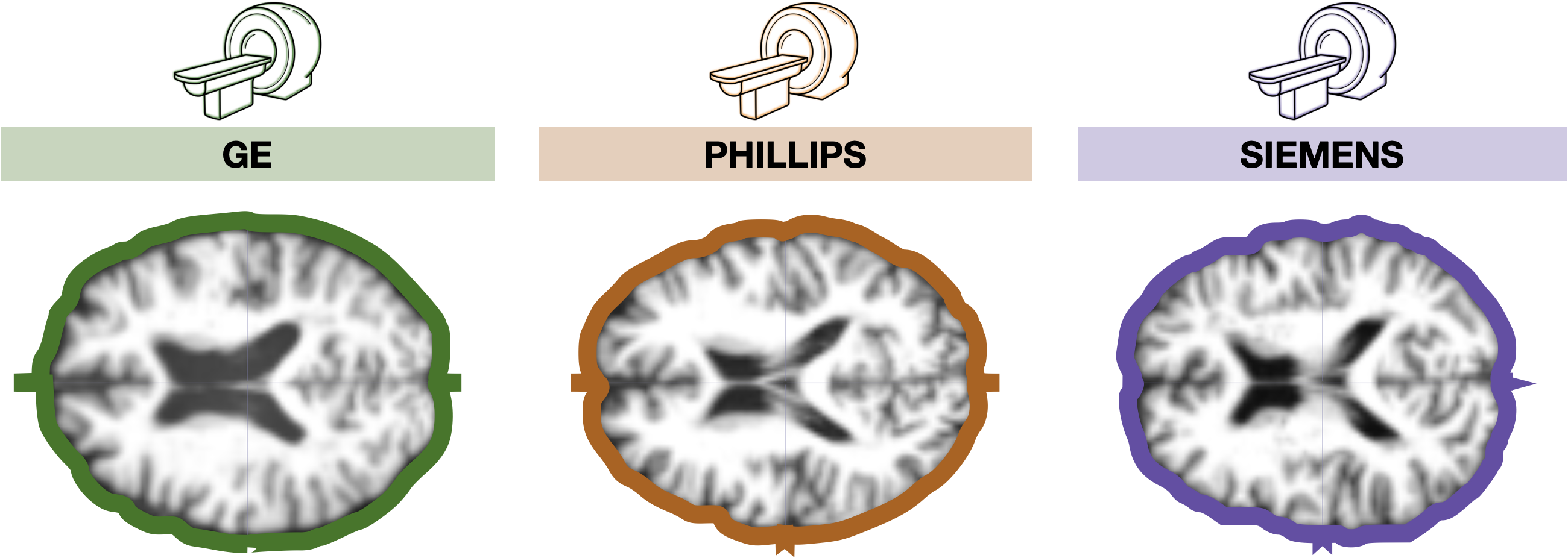}
		\caption{Variation due to site (scanner) for particular age group.}
		\label{fig:confounds_site}
	\end{subfigure} \hspace{4mm}%
	\begin{subfigure}[b]{0.37\linewidth}
		\includegraphics[width=\linewidth]{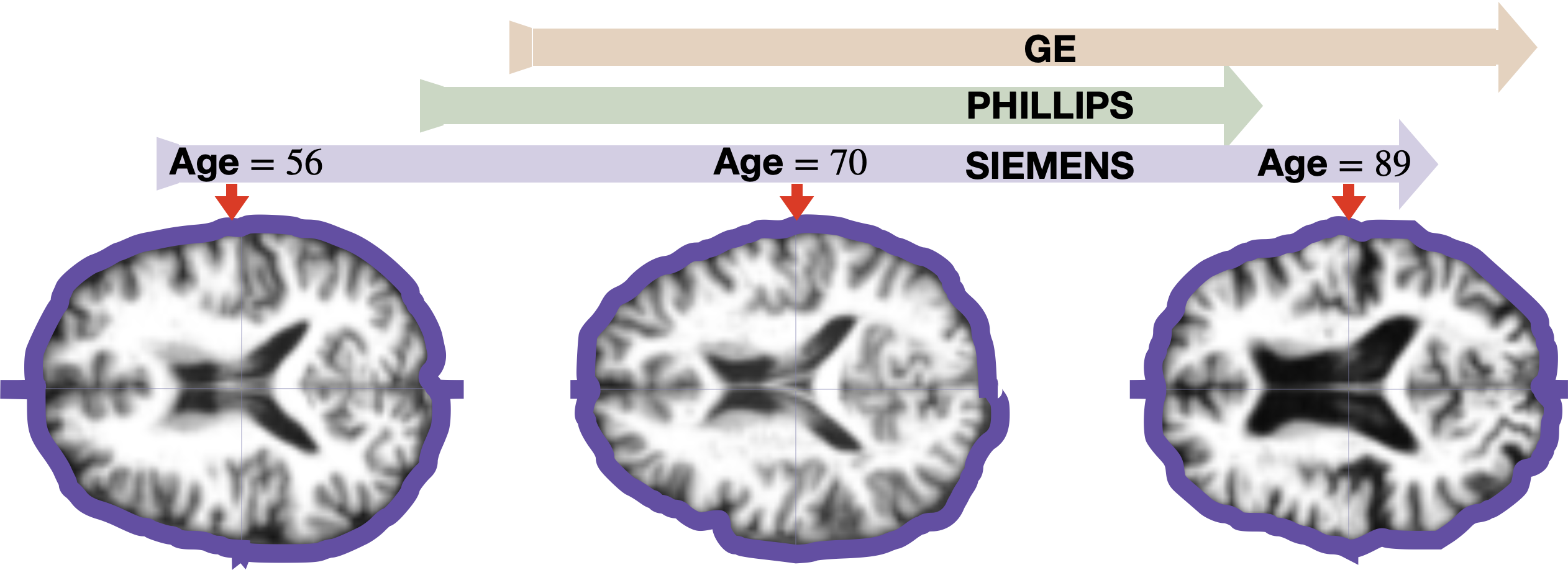}
		\caption{Variation due to covariates (age) in Siemens scanner.}
		\label{fig:confounds_covariate}
	\end{subfigure}%
	\centering
	\caption{\footnotesize \label{fig:scm_and_variations} \textbf{(a)} A \textbf{Causal Diagram } listing variable of interest and their relationship for multi-site pooling problem. Nodes $D_{\text{popul}}$, $D_{\text{acqui}}$ and $D_\text{preval}$ denote the population, acquisition and prevalence biases that vary across sites. $C$'s are covariates (like age or gender). $X_\mathscr{C}$ denotes the high-level causal features of an image $X$ that influences the labels $Y$. Nodes in red $d$-separate the nodes in  blue and green.  \textbf{(b)} MRI images on control subjects from the ADNI \cite{jack2008alzheimer} dataset for different \textbf{scanners} in the age group $70$-$80$.  \textbf{(c)} Images obtained from the Siemens scanner (i.e., fixing site) on control subjects for three extreme \textbf{age groups}. The gantt chart on top of the image indicates the respective age range in the Phillips and GE scanners. As observed,  different scanner groups do not share a common support on ``age" covariates, resulting in samples outside of the common support to be discarded in na\"ive pooling approaches.  }
		\vspace{-0.2in}
\end{figure*}

Over the last decade, there are coordinated 
large scale multi-institutional imaging studies 
(e.g., ADNI \cite{jack2008alzheimer}, NIH All of Us and HCP \cite{glasser2013minimal}) but the types of data collected 
or the project's scope 
(e.g., demographic pool of participants) may not be suited for 
studying specific 
{\em secondary} scientific questions. 
A ``pooled'' imaging dataset obtained from 
combining roughly 
similar studies across different institutions/sites, 
when possible, is an attractive 
alternative. The pooled datasets 
provide much larger sample sizes 
and improved statistical power 
to identify early disease bio-markers -- analyses 
which would not otherwise be possible 
\cite{luo2020sequence,fortin2017harmonization}.
But even when study participants are 
consistent across sites, pooling poses challenges. This is true even for linear regression \cite{pmlr-v70-zhou17c} -- 
improvement in statistical power is not always guaranteed. Partly due to these 
as well as other reasons, high visibility projects such as ENIGMA \cite{thompson2014enigma} have reported findings using  
meta-analysis methods. \\
{\bf Data pooling and fairness.} Even under ideal 
conditions, pooling imaging datasets across sites 
requires care. Assume that the participants 
across two sites, say $\text{site}_1$
and $\text{site}_2$, are perfectly gender matched with the same proportion 
of male/female and the age distribution 
(as well as the proportion of diseased/health controls) is also identical. 
In this idealized setting, 
the only difference between sites 
may come from variations in scanners or 
acquisition (e.g., pulse sequences). When 
training modern neural 
networks for 
a regression/classification 
task with imaging data obtained 
in this scenario, we may ask 
that the representations learned 
by the model be {\em invariant}
to the categorical variable 
denoting ``site''. While this is 
not a ``solved'' problem, this strategy has been successfully 
deployed based on 
results in invariant representation learning \cite{NEURIPS2018_415185ea,arjovsky2019invariant,akash2021learning} (see Fig. ~\ref{fig:encoder}).
One may alternatively view this task via the lens of 
fairness -- we want the model's 
performance to be fair with respect to the site variable. 
This approach is effective, via constraints \cite{zemel2013learning} or using adversarial modules \cite{zhang2018mitigating,ganin2016domain}. This setting also permits re-purposing tools from domain adaptation \cite{yang2021generalized,pandey2021domain,zhou2018statistical} or transfer learning \cite{dubey2021adaptive} as a pre-processing step, before analysis of 
the pooled data proceeds. \\
{\bf Nuisance variables/confounds.} Data pooling problems one often encounters in 
scientific research typically violates many of the 
conditions in the aforementioned example. 
The measured data $X$ at each site 
is influenced not only by the scanner properties 
but also by a number of other covariates / nuisance variables. 
For instance, if the age distribution of participants 
is not identical across sites, comparison 
of the site-wise distributions is challenging 
because it is influenced {\bf both} by age and the scanner. An example of the differences introduced due to age and scanner biases is shown in Figures~\ref{fig:confounds_site},~\ref{fig:confounds_covariate}. With multiple nuisance variables, even  effective tools for invariant representation 
learning, when used directly, 
can provide limited help. The data generation process, 
and the role of covariates/nuisance variables, available via a causal diagram (Figure~\ref{fig:scm}), can inform how the formulation is designed \cite{bareinboim2016causal,subbaswamy2019preventing}. Indeed, concepts 
from causality have benefited 
various deep learning models \cite{scholkopf2021toward,peters2017elements}. Specially, recent work \cite{mahajan2020domain} has shown 
the value of integrating structural causal models for 
domain generalization, which is  related 
to dataset pooling.\\
{\bf Causal Diagram. }Dataset pooling under completely arbitrary settings 
is challenging to study systematically. So, we assume that the site-specific imaging datasets are {\em not} significantly different to begin with, although the distributions for covariates 
such as age/disease prevalence may not be perfectly matched and each of these factors will influence the data. 
We assume access to a causal diagram describing 
how these variables influence the measurements. 
We show how the 
distribution matching criteria provided 
by a causal diagram can be nicely handled for some 
ordinal covariates that are not perfectly 
matched across sites by adapting 
ideas from equivariant 
representation learning. \\   
{\bf Contributions.} We propose a method to pool multiple neuroimaging datasets by learning representations that are robust to site (scanner) and covariate (age) values (see Fig.~\ref{fig:encoder} for visualization). We show that continuous 
nuisance covariates which do not have the same support and are not identically distributed across sites, can be 
effectively handled 
when learning invariant representations. 
We do not 
require finding ``closest match'' participants across sites -- a strategy loosely based on 
 covariate matching  \cite{rosenbaum1985constructing} from 
statistics which is less 
feasible if the 
distributions for a covariate (e.g., age) do not closely overlap.
Our model is based on adapting recent 
results on 
equivariance together with 
known 
concepts from group theory. 
When 
tied  
with common 
invariant representation learners, 
our formulation allows far improved 
analysis of pooled imaging datasets. We first 
perform evaluations on common 
fairness datasets and then show its applicability 
on two separate neuroimaging tasks 
with multiple nuisance variables.

\section{Reducing Multi-site Pooling to Infinite Dimensional Optimization}
\label{sec:method_intro}


Let $X$ denote an image 
of a participant 
and let $Y$ be the corresponding 
(continuous or discrete) response variable or target label (such as cognitive score or disease status). For simplicity, consider only two sites -- site$_1$
and site$_2$. Let $D$ represent the site-specific shifts, biases or covariates that we want to take into account. 
One possible data generation process relating these variables is shown in Figure~\ref{fig:scm}. 

{\bf Site-specific biases/confounds.} 
Observe that 
$Y$ is, in fact, influenced 
by high-level (or latent) features $X_\mathscr{C}$ specific to the participant. The images (or image-based disease biomarkers) $X$ are simply 
our (lossy) measurement of 
the participant's brain   
$X_\mathscr{C}$ \cite{eskildsen2015structural}. Further, $X$ also includes 
an (unknown) confound: contribution from the scanner (or acquisition protocol).  Figure~\ref{fig:scm} also lists covariates $C$, such as age and other factors which impact $X_\mathscr{C}$ (and 
therefore, $X$). A few common site-specific biases $D$ are shown in Fig.~\ref{fig:scm}. 
These 
include 
\begin{inparaenum}[\bfseries (i)]
\item \textit{ population bias }$D_{\rm popul}$ that leads to differences in age or gender distributions of the cohort \cite{castro2020causality}; 
\item we must also account for \textit{acquisition shift }$D_{\rm acqui}$ resulting from different scanners or imaging protocols -- this 
affects $X$ but not $X_\mathscr{C}$;  
\item data are also 
influenced by a \textit{class prevalence bias }$D_{\rm preval}$, e.g.,  healthier individuals over-represented in site$_2$ 
will impact the distribution of 
cognitive scores across sites. 
\end{inparaenum}

For imaging data, in 
principle, 
site-invariance can be achieved via an encoder-decoder style architecture to map the images $X$ into a
``site invariant'' latent space $\Phi(X)$. 
Here, $\Phi(X)$ in the idealized setting, 
corresponds to the true 
``causal'' features $X_\mathscr{C}$ that is comparable across sites. 
In practice, 
we know that images cannot 
fully capture the disease -- so, $\Phi(X)$
is simply a surrogate, limited by the 
measurements we have on hand. 
Given these caveats, 
an architecture is shown in Fig.~\ref{fig:encoder}. Ideally, the encoder will minimize Maximum Mean Discrepancy (MMD) \cite{gretton2006kernel} or another discrepancy 
between the distributions 
of latent representations $\Phi(\cdot)$ of the data from site$_1$ and site$_2$. 

{The site-specific attributes $D$ are often {\bf unobserved} or otherwise 
unavailable. For instance, we may 
not have full access to 
$D_{\rm popul}$ from which our 
participants are drawn. 
To tackle these issues, we use a {causal diagram}, see Fig.~\ref{fig:scm}, similar to existing works \cite{mahajan2020domain,zhou2018statistical} with minimal changes. For dealing with unobserved $D$'s, some standard approaches are known \cite{hardtrecht}. Let us see how it can help here. 
Applying $d$-separation (see \cite{pearl2016causal,hardtrecht} ) on  Fig.~\ref{fig:scm}, we see that the nodes $(D_{\text{popul}}, C, X_{\mathscr{C}})$ form a so-called ``head-to-tail'' branch and the nodes $(D_{\text{acqui}}, X, X_{\mathscr{C}})$, $(D_{\text{preval}}, Y, X_{\mathscr{C}})$ form a ``head-to-head'' branch. This implies that $X_\mathscr{C} \perp\!\!\!\perp D \mid C$. This is exactly an invariance condition: $X_\mathscr{C}$ should not change across different sites for samples with the {\bf same value} of $C$. 
To enforce this using $\Phi(\cdot)$, we must 
optimize a discrepancy between site-wise $\Phi(X)$'s at a given value of $C$, }
%
%
%
\begin{align}
	\label{eq:anchor2}
	\min_{ \Phi} \mathcal{MMD}\bigg(P_{\text{site}_1}\Big(\Phi(X) \mid C\Big), P_{\text{site}_2}\Big(\Phi(X) \mid C\Big)\bigg)
\end{align}
{
 
{\bf Provably solving \eqref{eq:anchor2}?} A brief comment on the difficulty 
 of the distributional optimization in \eqref{eq:anchor2} is useful. 
 Generic tools for (worst case) convergence rates for such problems are actively being developed \cite{yang2020variational}. For the average case, \cite{qi2020online} presents an online method for a specific class of  (finite dimensional) distributionally robust optimization problems that can be defined using standard divergence measures. Observe that even these convergence guarantees are local in nature, i.e., they output a point that satisfies necessary conditions and may not be sufficient. 
 }

In practice, the outlook is a little better. Intuitively, an optimal {\em matching} of the conditional distributions $P(\Phi(X) \mid C)$ across the two sites corresponds to a (globally) optimal solution to the probabilistic optimization task in \eqref{eq:anchor2}. 
Existing works show that 
it is indeed possible to approach 
this {\em computationally}
via sub-sampling methods \cite{zhou2018statistical} or by learning elaborate matching functions to identify image or object pairs across sites 
that are ``similar'' \cite{mahajan2020domain} or 
have the same value for $C$. Sub-sampling, 
by definition, reduces the number of samples from the two sites by discarding samples outside of the common support.
This impacts the quality of the estimator -- 
for instance, \cite{zhou2018statistical} must restrict 
the analysis only to that age range of $C$ which overlaps or is shared across sites. Such discarding of samples 
is clearly undesirable when each 
image acquisition is expensive.  { Matching functions also do not work if the support of $C$ is not identical across the sites, as briefly described next. \begin{exmp}
		Let $C$  denote 
		an observed covariate, e.g., age.
		Consider $X_i$ at $\text{site}_1$ 
		with $C=c_1$ and 
		$X_j$ at $\text{site}_2$ 
		with $C=c_2$. 
		If $c_1 \approx c_2$, 
		a matching will seek 
		$\Phi(X_i) \approx \Phi(X_j)$ in $X_\mathscr{C}$ space. If $c_1$ falls outside the support of 
		$c$'s acquired at $\text{site}_2$, one must 
		not only estimate $\Phi(\cdot)$ but also a transport expression $\Gamma_{c_2\to c_1}(\cdot)$ on 
		$X_\mathscr{C}$ such that 
		$\Phi(X_i) \approx \Gamma_{c_2\to c_1} (\Phi(X_j))$. The ``transportation'' involves estimating 
		what a latent image 
		acquired at age $c_2$ would 
		look like at age $c_1$. This means that matching would need a solution to the key difficulty, obtained further upstream.
\end{exmp}}


\begin{figure*}[!t]
	\includegraphics[width=0.7\textwidth]{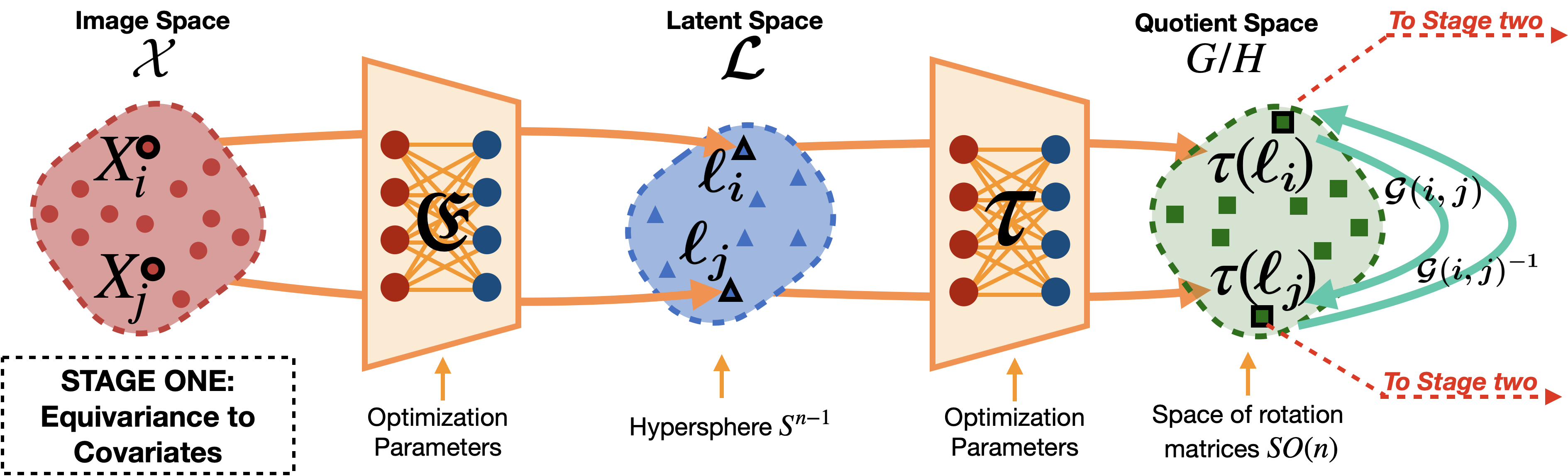}
	\centering
	\caption{\footnotesize {\bf Visualization of Stage one.}	\label{fig:tau_algo}  First, an image pair $X_i , X_j$ are mapped onto a hypersphere using an encoder $\mathfrak{E}$. The resulting pair $\lat_i , \lat_j$  are passed through  $\tau$ network to map them into the space of rotation matrices (which is the quotient group denoted by $G/H$). Fact~\ref{fact_eqv} ensure that $\tau$ is a $G=SO(n)-$equivariant map. $\mathcal{G}({i, j}) / \mathcal{G}({i, j})^{-1}$ is the group action of transforming $\tau(\lat_i)$ to $\tau(\lat_j) /  \tau(\lat_j)$ to $\tau(\lat_i)$ respectively.  } 
	\vspace{-0.2in}
\end{figure*}

\subsection{Improved Distribution Matching via Equivariant Mappings may be possible}

Ignoring $Y$ for the moment, recall that matching here corresponds to a bijection between unlabeled (finite) conditional distributions. Indeed, if the conditional distributions take specific forms such as a Poisson process, it is indeed possible to use simple matching algorithms that only require access to pairwise ranking information on the corresponding empirical distributions \cite{shah2017simple}, for example, the well-known Gale-Shapley algorithm \cite{teo2001gale}. Unfortunately, in applications that we consider here, such distributional assumptions may not be fully faithful with respect to site specific covariates $C$.  In essence,  we want representations $\Phi$ (when viewed as a function of $C$) that vary in a predictable (or say deterministic) manner 
 -- if so, we can 
avoid matching altogether and instead 
match a suitable property of the site-wise distributions of the representation $\Phi(X)$. 
We can make this criterion more specific. 
We want the site-wise distributions 
to vary in a manner where the ``trend'' is consistent across the sites. Assume that 
this were not true, say $P(\Phi(X)\mid C)$ is continuous and monotonically increases with $C$ for site$_1$ but monotonically decreases for site$_2$. A match of $P(\Phi(X)\mid C)$ across the sites at a particular value of $C=c$ implies at least one $C=c^\prime$ where $P(\Phi(X)\mid C)$ do not match. The monotonicity argument 
is weak for high dimensional 
$\Phi$. Plus, we have multiple nuisance 
variables. It turns out that our requirement for  $P(\Phi(X)\mid C)$ to vary in a predictable manner across sites can be handled using the idea of equivariant mappings, i.e., $P(\Phi(X)\mid C)$ must be {\bf equivariant} with respect to $C$ for both  sites. In addition, we will also seek 
{\bf invariance} to scanner attributes.

{ 
While we are familiar with the well-studied notion of invariance through the MMD criterion \cite{li2014learning}, we will briefly formalize our idea behind an equivariant mapping which is less common in this setting.
\begin{definition}
	\label{def:equivariance}
	A mapping $f: \mathcal{X} \rightarrow \mathcal{Y}$ defined over measurable Borel spaces $\mathcal{X}$ and $\mathcal{Y}$ is said to be $G-$equivariant under the action of group $G$ iff 
	$$ f(g \cdot x) = g \cdot f(x), \quad g \in G$$
\end{definition}
We refer the reader to two recent surveys, Section~$2.1$ of \cite{bloem2020probabilistic} and Section~$3.1$ of \cite{bronstein2021geometric}, which provide a detailed review.}

{ 
Equivariance is often understood in the context of a group action (say, a matrix group) \cite{hunacek2005lie,knapp1996lie}. While the covariates $C$ is a vector (and every vector space is an abelian group), 
since this group will eventually act on the latent space of our images, imposing 
additional structure will be beneficial. 
To do so, we will utilize 
a mapping between $C$'s and a 
group suitable for our setting. Once this is accomplished, we will derive an equivariant encoder. We discuss these
steps next. 
}

\section{Methods}
\label{sec:methods}
The dual goals of 
\begin{inparaenum}[\bfseries (i)]
\item equivariance to covariates (such as age)  and 
\item invariance to site, 
\end{inparaenum}
involves learning multiple mappings. For simplicity, and to keep the computational effort manageable, we divide our method into two stages. Briefly, our stages are 
\begin{inparaenum}[\bfseries (a)] 
	\item \textbf{Stage one: Equivariance to Covariates. } We learn a mapping to a space that provides the essential flexibility to characterize changes in covariate $C$ as a group action.  This enables us to construct a space satisfying the equivariance condition as per Def.~\ref{def:equivariance}
	\item \textbf{Stage two: Invariance to Site. } 
	We learn a second encoding to a generic vector space by apriori ensuring that the equivariance properties from Stage one are preserved. Such an encoding is then tuned to optimize the MMD criterion, thus generating a latent space that is invariant to site while remaining equivariant to covariates. 
\end{inparaenum}
We describe these stages one by one in the following sections.  



\subsection{Stage one: Equivariance to Covariates}
\label{sec:stage_one}
Given the space of images, $\mathcal{X}$,  with the covariates $C$, first, we want to characterize the effect of $C$ on $\mathcal{X}$ as a group action for some group $G$.  Here, an element $g\in G$ characterizes the change from covariate $c_i\in C$ to $c_j\in C$ (for short, we will use $i$ and $j$). The change in $C$ corresponds to a translation action which is difficult to instantiate in  $\mathcal{X}$ without invoking expensive conditional generative models. Instead, we propose to learn a mapping to a latent space $\mathcal{L}$ such that the change in $C$ can be characterized by a group action pertaining to $G$ in the space $\mathcal{L}$ (the latent space of $\mathcal{X}$). As an example, let us say $X_i$ goes to $X_j$ in $\mathcal{X}$ as $(X_i \to X_j)$. This means that $(X_i \to X_j)$ is caused due to the covariate change $(c_i \to c_j)$ in $C$. Let $\mathfrak{E}$ be a mapping between the image space $\mathcal{X}$ and the latent space $\mathcal{L}$. In the latent space $\mathcal{L}$, for the moment, we want that $(\mathfrak{E}X_i \to \mathfrak{E}X_j)$ should correspond to the change in covariate 
$(c_i \to c_j)$.

{ 
\begin{remark}
	\label{lab:key_difficulty}We are mostly interested in normalized covariates for example, in $\ell_p$ norm, while other volume based deterministic normalization functions may also be applicable. In the simplest case of $p=2$ norm, the corresponding group action is naturally induced by the matrix group of rotations.
\end{remark}
Based on this choice of group, we will learn an autoencoder $\left(\mathfrak{E}, \mathfrak{D}\right)$ with an encoder $\mathfrak{E}: \mathcal{X} \rightarrow \mathcal{L}$ and a decoder $\mathfrak{D}: \mathcal{L} \rightarrow \mathcal{X}$, here $\mathcal{L}$ is the encoding space.}  Due to Remark~\ref{lab:key_difficulty}, we can choose $\mathcal{L}$ to be a hypersphere, $\mathbf{S}^{n-1}$, and $\left(\mathfrak{E}, \mathfrak{D}\right)$ as a hyperspherical autoencoder \cite{zhao2019latent}. Then, we can characterize the ``action of $C$ on $\mathcal{X}$'' as the action of $G$ on $\mathbf{S}^{n-1}$. That is to say that a covariate change (translation in $C$) is a change in angles on $\mathcal{L}$. This corresponds to a rotation due to the choice of our group $G$. Note that for $\mathcal{L} = \mathbf{S}^{n-1}$, $G$ is the space of $n\times n$ rotation matrices, denoted by  $\textsf{SO}(n)$, and the action of $G$ is well-defined.  What remains is to encourage the latent space $\mathcal{L}$ to be $G$-equivariant. We start with some group theoretic properties that will be useful. 

\subsubsection{Review: Group  theoretic properties of $\textsf{SO}(n)$ }
\label{sec:review}
Let $\textsf{SO}(n) = \left\{X\in \mathbf{R}^{n\times n} | X^TX = I_n, \text{det}(X) = 1\right\}$ be the group of $n\times n$  special orthogonal matrices. The group $\textsf{SO}(n)$ acts on $\mathbf{S}^{n-1}$ with the group action ``$\cdot$'' given by $g\cdot \lat \mapsto g \lat$, for $g\in \textsf{SO}(n)$ and $\lat\in \mathbf{S}^{n-1}$. Here we use $g\lat$ to denote the multiplication of matrix $g$ with $\lat$. Under this group action, we can identify $\mathbf{S}^{n-1}$ with the quotient space $G/H$ with $G = \textsf{SO}(n)$ and $H = \textsf{SO}(n - 1)$ (see Ch.~$3$ of \cite{Dummit_Foote_2004} for more details). 
Let $\tau: \mathbf{S}^{n-1} \rightarrow G/H$ be such an identification, i.e., $\tau(\lat) = gH$ for some $g\in G$. The identification $\tau$ is equivariant to $G$ in the following sense.
\begin{fact}
\label{fact_eqv}
Given $\tau: \mathbf{S}^{n - 1} \rightarrow G/H$ as defined above, $\tau$ is equivariant with the action of $G$, i.e., $\tau\left(g\cdot \lat\right) = g\tau(\lat)$.
\end{fact}

Next, we see that given two points $\lat_i, \lat_j$ on $\mathbf{S}^{n-1}$ there is a unique group element in $G$ to move from $\tau(\lat_i)$ to  $\tau(\lat_j)$.  
\begin{lemma}
\label{lemma_eqv}
Given two latent space representations $\lat_i, \lat_j \in \mathbf{S}^{n-1}$, and the corresponding cosets $g_iH = \tau(\lat_i)$ and $g_jH = \tau(\lat_j)$, $\exists! g_{ij} = g_jg_i^{-1}\in G$ such that $\lat_j = g_{ij}\cdot \lat_i$.
\end{lemma}
Thanks to Fact~\ref{fact_eqv} and Lemma~\ref{lemma_eqv}, simply identifying a suitable $\tau$ will provide us the necessary equivariance property. To do so, next, we parameterize $\tau$ by a neural network and describe a loss function to learn such a $\tau$ and $\left(\mathfrak{E}, \mathfrak{D}\right)$. 

\subsubsection{Learning a $G$-equivariant $\tau$ with DNNs }

Now that we established the key components:  \begin{inparaenum}[\bfseries (a)]
\item an  autoencoder $\left(\mathfrak{E}, \mathfrak{D}\right)$ to map from $\mathcal{X}$ to the latent space $\mathbf{S}^{n-1}$ \item a mapping $\tau: \mathbf{S}^{n-1} \rightarrow \textsf{SO}(n)$ which is $G = \textsf{SO}(n)$-equivariant,  \end{inparaenum} see Figure \ref{fig:tau_algo}, we discuss how to learn such a $\left(\mathfrak{E}, \mathfrak{D}\right)$ and a $G$-equivariant $\tau$. 

Let $X_i, X_j\in \mathcal{X}$ be two images with the corresponding covariates $i, j\in C$ with $i\neq j$. Let $\lat_i = \mathfrak{E}\left(X_i\right), \lat_j = \mathfrak{E}\left(X_j\right)$. Using Lemma \ref{lemma_eqv}, we can see that a $g_{ij}\in G$ to move from $\lat_i$ to $\lat_j$ does exist and is unique. Now, to learn a $\tau$ that satisfies the equivariance property (Fact~\ref{fact_eqv}), we will need $\tau$ to satisfy two conditions, $\tau(g_{ij} \cdot \lat_i) = g_{ij} \tau(\lat_i)$  and $\tau(g_{ji} \cdot \lat_j) = g_{ji} \tau(\lat_j)$ $\forall g \in G$. The two conditions are captured in the following loss function,
\begin{align}
	\lat_i &= \mathfrak{E}(X_i) \quad \lat_j = \mathfrak{E}(X_j)\\
	\label{loss_equv}
	L_{\text{stage} 1} = \hspace{-0.2in} \mathlarger\sum_{\substack{\left\{\left(X_i, i\right), \left( X_j, j\right)\right\} \\ \subset \mathcal{X}\times C} } \hspace{-0.24in} &\left.
	\begin{array}{ll}
		&\|\mathcal{G}(i, j)\cdot \tau\left(\lat_i\right) - \tau\left(\lat_j\right)\|^2 \quad +\\
		&\|\mathcal{G}^{-1}(i, j)\cdot \tau\left(\lat_j\right) - \tau\left(\lat_i\right)\|^2 
	\end{array}
	\right.
\end{align}
Here, $\mathcal{G}:C\times C\rightarrow G$ will be a table lookup given by $(i, j)\mapsto g_{ij}$ is the function that takes two values for the covariate $c$, say, $i, j$ corresponding to $X_i, X_j \in \mathcal{X}$ and 
simply returns the group element (rotation) $g_{ij}$ needed to move from $\mathfrak{E}(X_i)$ to $\mathfrak{E}(X_j)$. {\bf Choice of  $\mathcal{G}$:} In general, learning $\mathcal{G}$ is difficult since $C$ may not be continuous. In this work, we fix $\mathcal{G}$ and learn $\tau$ by minimizing \eqref{loss_equv}. We will simplify the choice of $\mathcal{G}$ as follows: assuming that $C$ is a numerical/ordinal random variable, we define $\mathcal{G}$ by $(i, j)\mapsto \textsf{expm}((i-j)\mathbf{1}_m)$. Here $m = {n \choose 2}$ is the dimension of $G$ and $\textsf{expm}$ is the matrix exponential, i.e., $\textsf{expm}: \mathfrak{so}(n) \rightarrow \textsf{SO}(n)$, where $\mathfrak{so}(n)$ is the Lie algebra \cite{hall2018theory} of $\textsf{SO}(n)$. Since $\mathfrak{so}(n)$ is a vector space, hence $(i-j)\mathbf{1}_m \in \mathfrak{so}(n)$. To reduce the runtime of $\textsf{expm}$, we replace $\textsf{expm}$ by a Cayley map \cite{selig2007cayley,mehta2019scaling} defined by: $\mathfrak{so}(n)\ni A \mapsto (I-A)(I+A)^{-1} \in \textsf{SO}(n)$. Here we used \textsf{expm} for parameterization (other choices also suitable). 

 Finally, we learn the encoder-decoder $(\mathfrak{E},\mathfrak{D})$ by using a reconstruction loss constraint with $L_{\text{stage} 1}$  in \eqref{loss_equv}. This 
 can also be thought of as a combined loss for this stage as $L_{\text{stage} 1} + \sum_i \| X_i - \mathfrak{D}(\mathfrak{E}(X_i))\|^2$ where the second term is the reconstruction loss.  The loss balances two terms and requires a scaling factor (see appendix \S~\ref{sec:app_scale}).  A flowchart of all steps in this stage can be seen in Fig~\ref{fig:tau_algo}. 
%

\begin{algorithm}[t]
	\caption{Learning representations that are \textit{Equivariant} to \textit{Covariates} and \textit{Invariant} to \textit{Site}}
	\label{alg:stage1_stage2}
	\begin{algorithmic}[0]
		\STATE \tikzmk{A}\textbf{Input:} Training Sets from multiple sites  $(X,Y)_{\text{site1}} $, $(X,Y)_{\text{site2}} $. Nuisance covariates $C$.\\
		\STATE \tikzmk{B}\boxit{yellow}\tikzmk{A}\textbf{Stage one: Equivariance to Covariates}
		\STATE $1:$ Parameterize Encoder-Decoder pairs $(\mathfrak{E},\mathfrak{D})$ and $\tau$ mapping with neural networks \\
		\STATE $2:$ Optimize over $(\mathfrak{E},\mathfrak{D})$ and $\tau$ to minimize,   \\
		\STATE \qquad \qquad $L_{\text{stage} 1} + \sum_i \| X_i - \mathfrak{D}(\mathfrak{E}(X_i))\|^2$
		\STATE {\em Output:} First latent space mapping $\mathfrak{E}$ and a supporting mapping function $\tau$. Here, $\tau$ is $G$-equivariant to the covariates $C$ (see Lemma~\eqref{lemma_eqv} and \eqref{loss_equv}). \\
		\STATE \tikzmk{B}\boxit{blue}\tikzmk{A}\textbf{Stage two: Invariance to Site}
		\STATE $1:$ Parameterize encoder $b$, predictor $h$ and decoder $\Psi$ with neural networks \\
		\STATE $2:$ Preserve equivariance from stage one with an equivariant mapping $\Phi$, (see Lemma~\eqref{arb_equiv}) \\
		\STATE $3:$ Optimize  $\Phi, b, h$ and $\Psi$ to minimize  $L_{\text{stage} 2} + \mathcal{MMD}$ \\
		\STATE {\em Output:} Second latent space mapping $\Phi$. Here, $\Phi$ is equivariant to the covariates and invariant to site.\\
		\tikzmk{B}\boxit{pink}
	\end{algorithmic}
\end{algorithm}

\begin{figure*}[!t]
	\begin{subfigure}[b]{0.4\linewidth}
		\includegraphics[width=\linewidth]{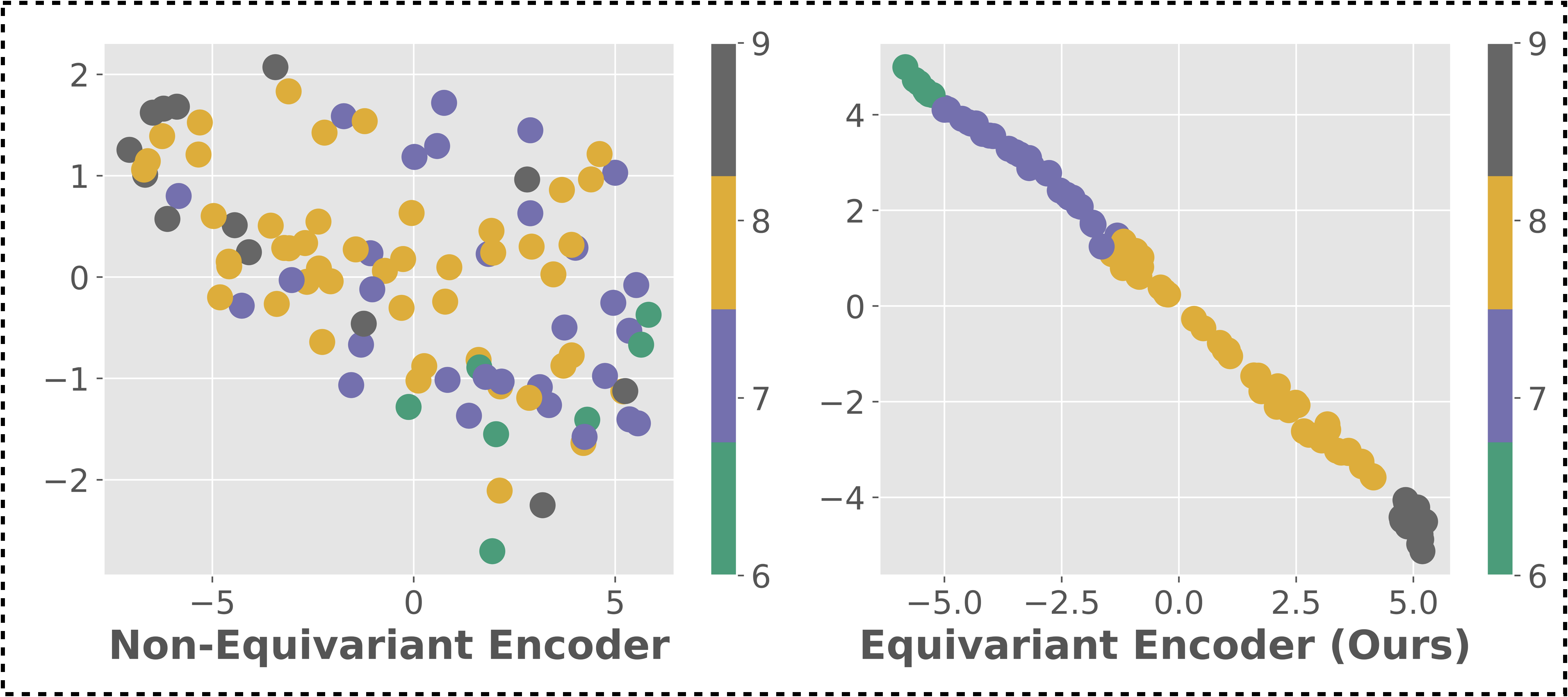}
		\caption{ADNI Dataset}
	\end{subfigure} \hspace{1cm}
	\begin{subfigure}[b]{0.4\linewidth}
		\includegraphics[width=\linewidth]{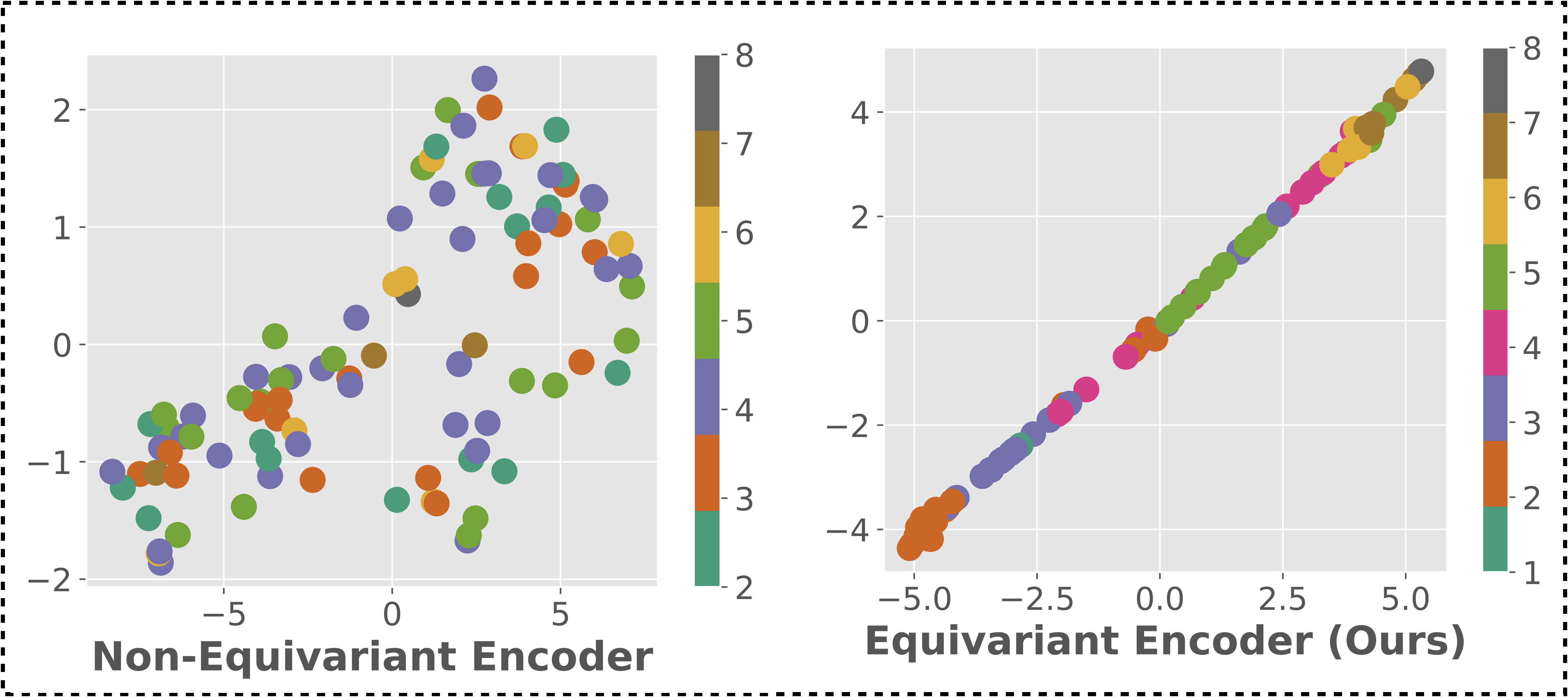}
		\caption{Adult Dataset}
	\end{subfigure}%
	\vspace{-0.1in}	
	\centering
	\caption{\label{fig:tsne_comparison} \footnotesize \textbf{ t-SNE plots of latent representations $\tau(\lat)$.} For ADNI \textbf{(left)} and Adult \textbf{(right)}, an equivariant encoder ensures that the latent features are evenly distributed and bear a monotonic trend with respect to the changes in the age covariate value. The non-equivariant space is generated from the Na\"ive pooling baseline. Each color denotes a discretized \textbf{age} group. Age was discretized only for the figure to highlight the density of samples in each age group.  }
	\vspace{-0.2in}
\end{figure*}

\subsection{Stage two: Invariance to Site}
\label{sec:stage_two}
Having constructed a latent space $\mathcal{L}$ that is equivariant to changes in the covariates $C$, we must now handle the site attribute,  
i.e., invariance with respect to site. 
Here, it will be convenient to project $\mathcal{L}$ onto a space that simultaneously preserves the equivariant structure from  $\mathcal{L}$ \textit{and} offers the flexibility to enforce site-invariance. The following lemma, inspired from the functional representations of probabilistic symmetries (\S$4.2$ of \cite{bloem2020probabilistic}), provides us 
strategies to achieve this goal. Here, consider $\Phi:\mathcal{L}\rightarrow \mathcal{Z}$ to be the projection.
	
\begin{lemma}
	\vspace{-0.1in}
\label{arb_equiv}
For a $\tau: \mathcal{L}\rightarrow G/H$ as defined above, and for any arbitrary mapping $b:\mathcal{L}\rightarrow \mathcal{Z}$, the function $\Phi: \mathcal{L}\rightarrow \mathcal{Z}$ defined by
\begin{align}
\label{eq_arb}
\Phi(\lat) = \tau(\lat)\cdot b\left(\tau(\lat)^{-1}\cdot \lat\right) 
\end{align}
is $G$-equivariant, i.e., $\Phi(g\cdot \lat) = g\Phi(\lat)$.
\end{lemma}
Proof is available in the appendix \S~\ref{sec:app_proofs}. Note that $\Phi$ remains equivariant for {\bf any mapping} $b$. This provides us the option to parameterize $b$ as a neural network and train the entirety of $\Phi$ for the desired site invariance where equivariance will be preserved due to \eqref{eq_arb}.  In this work, we learn such a $\Phi: \mathcal{L}\rightarrow \mathcal{Z}$ with the help of a decoder $\Psi: \mathcal{Z}\rightarrow \mathcal{L}$ by minimizing the following loss,   
\begin{align}
    \label{loss_arb_equv}
    \hspace{-0.1in} L_{\text{stage} 2} = \hspace{-0.1in} \sum_{\substack{\lat = \mathfrak{E}(X)\in \mathcal{L} \\ X\in \mathcal{X}, Y \in \mathcal{Y}}}  \overbrace{\|\lat - \Psi(\Phi(\lat))\|^2}^{\text{Reconstruction loss}}  +  \overbrace{\|Y-h(\Phi(\lat))\|^2}^{\text{Prediction loss}} \\ 
    \label{loss_arb_equv_constraint}
    \text{subject to} \quad \underbrace{\Phi(\lat) = \tau(\lat)\cdot b\left(\tau(\lat)^{-1}\cdot \lat\right)}_{\text{$G$-equivariant map}}
\end{align}
Minimizing the loss (\ref{loss_arb_equv}) with the constraint (\ref{loss_arb_equv_constraint}) allows learning the network $b: \mathcal{L} \rightarrow \mathcal{Z}$ and the decoder $\Psi:\mathcal{Z}\rightarrow \mathcal{L}$. 
%
We are now left with asking that $Z\in \mathcal{Z}$ be such that the representations are invariant across the sites. We simply use the following MMD criterion although other statistical distance measures can also be utilized. 
\begin{align}
\label{mmd_eq}
    \hspace{-0.1in} \mathcal{MMD} = \|\underset{\substack{Z_1\sim \\ P\left(\Phi(\lat)\right)_{\text{site1}}}}{\E} \mathcal{K}(Z_1, \cdot) - \underset{\substack{Z_2\sim \\P\left(\Phi(\lat)\right)_{\text{site2}}}}{\E} \mathcal{K}(Z_2, \cdot)\|_{\mathcal{H}}
\end{align}
The criterion is defined using a Reproducing Kernel Hilbert Space with norm $\|\cdot\|_{\mathcal{H}}$ and kernel $\mathcal{K}$. We combine (\ref{loss_arb_equv}), (\ref{loss_arb_equv_constraint}) and (\ref{mmd_eq}) as the objective function to ensure site invariance. Thus, the combined loss function $L_{\text{stage} 2}+ \mathcal{MMD}$ is minimized to learn $(\Phi, \Psi)$. Scaling factor details are available in the appendix \S~\ref{sec:app_scale}.

{
{\bf Summary of the two stages.}
 Our overall method comprises of two stages. The first stage, Section~\ref{sec:stage_one},  involves learning the  $\tau$ function. The  function learned in this stage is  $G$-equivariant by the choice of the loss $L_{\text{stage} 1}$, see \eqref{loss_equv}. Our next stage, Section~\ref{sec:stage_two}, employs the learned $\tau$ function and a trainable mapping $b$ to generate invariant representations. This stage preserves $G$-equivariance due to the $\Phi$ mapping in \eqref{eq_arb}. The loss for the second step is $L_{\text{stage} 2} + \mathcal{MMD}$ , see \eqref{loss_arb_equv}. Our method is summarized in Algorithm~\ref{alg:stage1_stage2}. Convergence behavior of the proposed optimization (of $\tau,\Phi$) still seems challenging to characterize exactly, but recent papers provide some hope, and opportunities. For example, if the networks are linear, then results from \cite{ghosh2020alternating} maybe applicable which explain the our superior empirical performance. 
}

\section{Experiments}
We evaluate our proposed encoder for site-invariance and robustness to changes in the covariate values $C$. Evaluations are performed on two multi-site neuroimaging datasets, where algorithmic developments are likely to be most impactful. Prior to neuroimaging datasets, we also conduct experiments on two standard fairness datasets, German and Adult. The inclusion of fairness datasets in our analysis, provides us a means for sanity tests and optimization feasibility on an established problem. Here, the goal of achieving fair representations is treated as pooling multiple subsets of data indexed by separate sensitive attributes. We begin our analysis by first describing our measures of evaluation and then reporting baselines for comparisons. 


{\bf Measures of Evaluation.}  Recall that our method involves learning $\tau$ as in \eqref{loss_equv} to satisfy the equivariance property. Moreover, we need to learn $\Phi$ as in \eqref{eq_arb}--\eqref{loss_arb_equv} to achieve site invariance. Our measures assess the structure of the latent space $\tau(\lat)$ and $\Phi(\lat)$. The measures are: \begin{inparaenum}[\bfseries (a)]\item $\boldsymbol{\Delta_{Eq}}:$ This metric evaluates the $\ell_2$ distance between $\tau(\lat_i)$ and  $\tau(\lat_j)$ for all pairs $i, j$. Formally, it is computed as 
\begin{align}
	\label{loss_equv1}
	\boldsymbol{\Delta_{Eq}} &= \hspace{-0.1in}\sum_{\substack{\left\{\left(X_i, i\right), \left( X_j, j\right)\right\}\subset \mathcal{X}\times C\\ \lat_i = \mathfrak{E}_e(X_i), \lat_j = \mathfrak{E}_e(X_j)} } |i-j| \| \tau\left(\lat_i\right) - \tau\left(\lat_j\right)\|^2 
\end{align}
 A higher value of this metric indicates that  $\tau(\lat_i)$ and  $\tau(\lat_j)$ are related by the group action $g_{ij}$. Additionally, we use t-SNE \cite{van2008visualizing} to qualitatively visualize the effect of $\tau$.
 \item $\boldsymbol{Adv}:$ This metric quantifies the site-invariance achieved by the encoder $\Phi$. We evaluate if $\Phi(\lat)$ for a learned $\lat\in \mathcal{L}$ has any information about the site. A three layered fully network (see appendix \S~\ref{sec:app_nn}) is trained as an adversary to predict site from $\Phi(\lat)$, similar to \cite{NIPS2017_8cb22bdd}. A lower value of $Adv$, that is close to random chance, is desirable. \item $\boldsymbol{\mathcal{M}}:$ Here, we compute the $\mathcal{MMD}$ measure, as in (\ref{mmd_eq}), on the test set. A smaller value of $\mathcal{M}$ indicates better invariance to site. Lastly, \item $\boldsymbol{\mathcal{ACC}}:$ This metric notes the test set accuracy in predicting the target variable $Y$.  \end{inparaenum} 
 
 {\bf Baselines for Comparison.} We contrast our method's performance with respect to a few well-known baselines. \begin{inparaenum}[(i)] \item  \textbf{Na\"ive:} This method indicates a na\"ive approach of pooling data from multiple sites without any scheme to handle nuisance variables. \item   \textbf{MMD \cite{li2014learning}:} This method minimizes the distribution differences across the sites without any requirements for equivariance to the covariates. The latent representations being devoid of the equivariance property result in lower accuracy values as we will see shortly. \item \textbf{CAI \cite{NIPS2017_8cb22bdd}:}  This method introduces a discriminator to train the encoder in a minimax adversarial fashion. The training routine directly optimizes the $Adv$ measure above. While being a powerful implicit data model, adversarial methods are known to have unstable training and lack convergence guarantees \cite{NEURIPS2019_56c51a39}.  \item \textbf{SS \cite{zhou2018statistical}:} This method adopts a Sub-sampling (SS) framework to divide the images across the sites by the covariate values $C$. An MMD criterion is minimized individually for each of the sub-sampled groups and an average estimate is computed. Lastly, \item \textbf{RM \cite{motiian2017unified}:} Also used in \cite{mahajan2020domain}, RandMatch (RM) learns invariant representations on samples across sites that "match" in terms of the class label (we match based on both $Y$ and $C$ values)  \end{inparaenum}. Below, we summarize each method and nuisance attribute correction adopted by them.
 
 \begin{table}[!h]
 	\vspace{-0.1in}
 	\centering
 	\resizebox{\columnwidth}{!}{%
 		\begin{tabular}{c c c c c c c} \hline \hline
 			\multirow{1}{*}{Correction}   & Na\"ive             & MMD \cite{li2014learning}  &  CAI \cite{NIPS2017_8cb22bdd} & SS \cite{zhou2018statistical}  &  RM \cite{motiian2017unified} & Ours \\ \hline\hline \\
			\multirow{1}{*}{Site } & \multirow{1}{*}{\xmark} & \multirow{1}{*}{\cmark} & \multirow{1}{*}{\cmark}      & \multirow{1}{*}{\cmark}  & \multirow{1}{*}{\cmark} & \multirow{1}{*}{\cmark}\\
 			\multirow{1}{*}{Covariates}     & \multirow{1}{*}{\xmark} & \multirow{1}{*}{\xmark} & \multirow{1}{*}{\xmark}  & \multirow{1}{*}{\cmark} & \multirow{1}{*}{\cmark} & \multirow{1}{*}{\cmark}  \\
 			\hline\hline
 		\end{tabular}%
 	}
 	\caption{\label{saas_hp_table} \footnotesize Baselines in the paper and their nuisance attribute correction. }
 	\vspace{-0.1in}
 \end{table}

 \begin{figure}[!t]
	\includegraphics[width=\columnwidth]{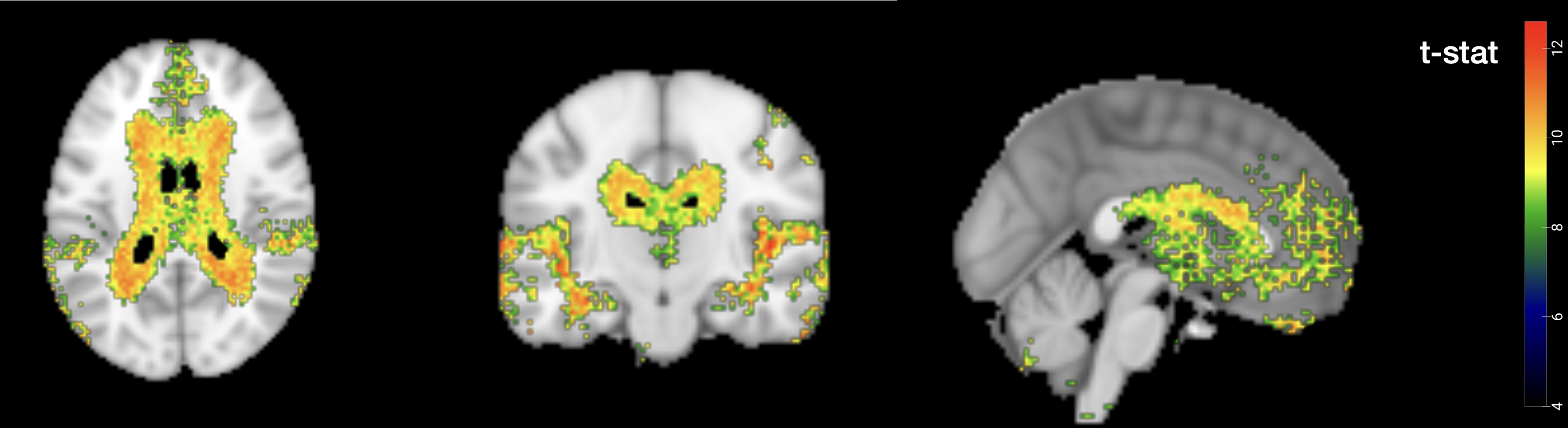} \quad
	\centering
	\vspace{-0.2in}
	\caption{\footnotesize \textbf{Statistical Analysis on the reconstructed outputs.} The voxels that are significantly associated with Alzheimer's disease  (p $< 0.001$) are shown. Adjustments for multiple comparisons were made using Bonferroni correction.  A high density of significant voxels indicates that our method preserves disease related signal after pooling across scanners.}
	\label{fig:tstatmap} 
	\vspace{-0.2in}
\end{figure}
 
 We evaluate methods on the test partition provided with the  datasets. The mean of the metrics over three random seeds is reported. The hyper-parameter selection is done on a validation split from the training set, such that the prediction accuracy falls within $5\%$ window relative to the best performing model \cite{NEURIPS2018_donini} (more details in appendix \S~\ref{sec:app_evaluation}). 
\subsection{Obtaining Fair Representations}
We approach the problem of learning fair representations through our multi-site pooling formulation. Specifically, we consider each sensitive attribute value as a separate site. Results on two benchmark datasets, German and Adult \cite{dua2017uci}, are described below.

{\bf German Dataset.}
This dataset is a classification problem used to predict defaults on the consumer loans in the German market.  Among the several features in the dataset, the attribute \textit{foreigner} is chosen as a sensitive attribute. We train our encoder while maintaining equivariance with respect to the continuous valued \textit{age} feature. Table~\ref{tab:results} provides a summary of the results in comparison to the baselines. Our equivariant encoder maximizes the $\Delta_{Eq}$ metric indicating the the latent space $\tau(\lat)$ is well separated for different values of \textit{age}. Further, the invariance constraint improves the $Adv$ metric signifying a better elimination of sensitive attribute information from the representations. The $\mathcal{M}$ metric is higher relative to the other baselines. The $\mathcal{ACC}$ for all the methods are within a $2\%$ range.

{\bf Adult Dataset.}
In the Adult dataset, the task is to predict if a person has an income higher (or lower) than $\$50K$ per year. The dataset is biased with respect to gender, roughly, $1$-in-$5$ women (in contrast to $1$-in-$3$ men) are reported to make over $\$50K$. Thus, the female/male genders are considered as two separate sites with \textit{age} as a nuisance covariate feature. As shown in Table~\ref{tab:results}, our equivariant encoder improves on metrics $\Delta_{Eq}$ and $Adv$ relative to all the baselines similar to the German dataset. In addition to the quantitative metrics, we visualize the t-SNE plots of the representations $\tau(\lat)$ in Fig.~\ref{fig:tsne_comparison} (right). It is clear from the figure that an equivariant encoder imposes a certain monotonic trend as the \textit{Age} values as varied. 

\begin{figure}[!b]
	\vspace{-0.1in}
	\includegraphics[width=\columnwidth]{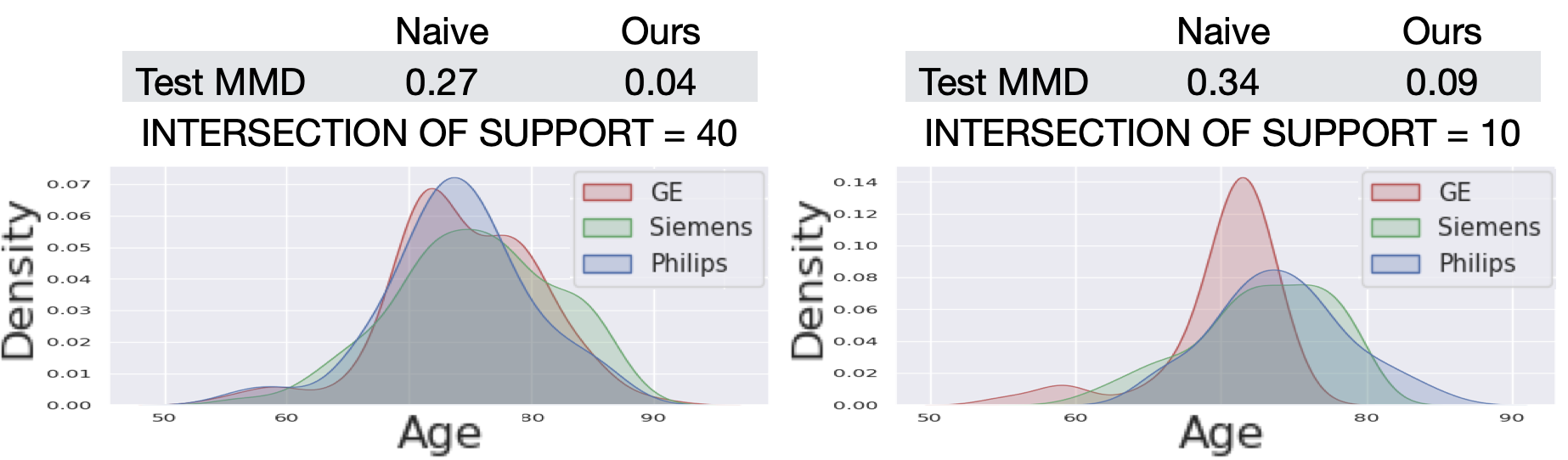} \quad
	\centering
	\vspace{-0.2in}
	\caption{\footnotesize \textbf{Distribution of age covariate in the ADNI dataset.} Two settings are considered -- \textbf{(left)} the intersection of the support is large, and \textbf{(right)} with a smaller common support. Despite the mismatch of support across scanner attributes, our approach minimizes the MMD measure (desirable) on the test set relative to the na\"ive pooling method.}
	\label{fig:density_comparison} 
	\vspace{-0.2in}
\end{figure}

\begin{table*}[!t] 
	\footnotesize
	\captionsetup{justification=centering} 
	\setlength\tabcolsep{0pt} 
	\caption*{$\boldsymbol{{\Delta}_{Eq}}:$  Equivariance Gap, $\boldsymbol{Adv}:$ Adversarial Test Accuracy, $\boldsymbol{\mathcal{M}}:$ Test  $\mathcal{MMD}$ measure, $\boldsymbol{\mathcal{ACC}}:$ Test prediction accuracy\\
		$\uparrow$: Higher Value is preferred, $\downarrow$: Lower Value is preferred}
	\vspace{-4mm}
	\begin{tabular*}{\linewidth}{l @{\extracolsep{\fill}}
			*{24}{S[table-format=3.0]}}
		\midrule\midrule
		& \multicolumn{4}{c}{German} & \multicolumn{4}{c}{Adult} & \multicolumn{4}{c}{ADNI}& \multicolumn{4}{c}{ADCP} \\
		\cmidrule{2-5} \cmidrule{6-9} \cmidrule{10-13} \cmidrule{14-17}
		& {${\Delta}_{Eq}$ $\uparrow$} & {$Adv$ $\downarrow$} & {$\newsym{M}$ $\downarrow$} & \cellcolor{gray!25}{$\mathcal{ACC}$ $\uparrow$} & {${\Delta}_{Eq}$ $\uparrow$} & {$Adv$ $\downarrow$} & {$\newsym{M}$ $\downarrow$} & \cellcolor{gray!25}{$\mathcal{ACC}$ $\uparrow$} & {${\Delta}_{Eq}$ $\uparrow$} & {$Adv$ $\downarrow$} & {$\newsym{M}$ $\downarrow$} & \cellcolor{gray!25}{$\mathcal{ACC}$ $\uparrow$} & {${\Delta}_{Eq}$ $\uparrow$} & {$Adv$ $\downarrow$} & {$\newsym{M}$ $\downarrow$} & \cellcolor{gray!25}{$\mathcal{ACC}$ $\uparrow$}  \\ 
		\midrule\midrule
		Na\"ive & {$4.6{\scriptscriptstyle(0.7)}$ } &  {$0.62{\scriptscriptstyle(0.03)}$ }  &  {$7.7{\scriptscriptstyle(0.8)}$}  &  \cellcolor{gray!25}{$74{\scriptscriptstyle(0.9)}$}  &  {$3.4{\scriptscriptstyle(0.7)}$}  &  {$83{\scriptscriptstyle(0.1)}$}  &  {$9.8{\scriptscriptstyle(0.3)}$}  &  \cellcolor{gray!25}{$84{\scriptscriptstyle(0.1)}$}  &  {$3.1{\scriptscriptstyle(1.0)}$}  &  {$59{\scriptscriptstyle(2.9)}$}  &  {$27{\scriptscriptstyle(1.6)}$}  &  \cellcolor{gray!25}{$80{\scriptscriptstyle(2.6)}$} &  {$4.1{\scriptscriptstyle(0.9)}$} &  {$49{\scriptscriptstyle(8.4)}$} &  {$90{\scriptscriptstyle(8.7)}$} &  \cellcolor{gray!25}{$83{\scriptscriptstyle(4.4)}$}    \\
		MMD \cite{li2014learning}  & {$4.5{\scriptscriptstyle(1.0)}$ }&  {$0.66{\scriptscriptstyle(0.04)}$}  &  {$1.5{\scriptscriptstyle(0.3)}$}  &  \cellcolor{gray!25}{$73{\scriptscriptstyle(1.5)}$}  &  {$3.4{\scriptscriptstyle(0.9)}$}  &  {$83{\scriptscriptstyle(0.1)}$}  &  {$3.1{\scriptscriptstyle(0.3)}$}  &  \cellcolor{gray!25}{$84{\scriptscriptstyle(0.1)}$} &  {$3.1{\scriptscriptstyle(1.0)}$}  &  {$59{\scriptscriptstyle(3.3)}$}  &  {$27{\scriptscriptstyle(1.7)}$}  &  \cellcolor{gray!25}{$80{\scriptscriptstyle(2.6)}$}   &  {$3.6{\scriptscriptstyle(1.0)}$} &  {$49{\scriptscriptstyle(11.9)}$} &  {$86{\scriptscriptstyle(11.0)}$} &  \cellcolor{gray!25}{$84{\scriptscriptstyle(6.5)}$}  \\ 
		CAI \cite{NIPS2017_8cb22bdd} & {$1.9{\scriptscriptstyle(0.6)}$ }&  {$0.65{\scriptscriptstyle(0.01)}$}  &  {$1.2{\scriptscriptstyle(0.2)}$}  &  \cellcolor{gray!25}{$76{\scriptscriptstyle(1.3)}$}  &  {$0.1{\scriptscriptstyle(0.0)}$}  &  {$81{\scriptscriptstyle(0.7)}$}  &  {$4.2{\scriptscriptstyle(2.4)}$}  &  \cellcolor{gray!25}{$84{\scriptscriptstyle(0.04)}$}  &  {$2.4{\scriptscriptstyle(0.7)}$}  &  {$61{\scriptscriptstyle(2.1)}$}  &  {$27{\scriptscriptstyle(1.5)}$}  &  \cellcolor{gray!25}{$74{\scriptscriptstyle(3.6)}$}   &  {$2.8{\scriptscriptstyle(1.6)}$} &  {$56{\scriptscriptstyle(6.9)}$} &  {$85{\scriptscriptstyle(12.3)}$} &  \cellcolor{gray!25}{$82{\scriptscriptstyle(5.1)}$} \\ 
		SS \cite{zhou2018statistical} & {$3.8{\scriptscriptstyle(0.5)}$ }&  {$0.70{\scriptscriptstyle(0.07)}$}  &  {$1.5{\scriptscriptstyle(0.6)}$}  &  \cellcolor{gray!25}{$76{\scriptscriptstyle(0.9)}$}  &  {$2.8{\scriptscriptstyle(0.5)}$}  &  {$83{\scriptscriptstyle(0.2)}$}  &  {$1.5{\scriptscriptstyle(0.2)}$}  &  \cellcolor{gray!25}{$84{\scriptscriptstyle(0.1)}$}  &  {$3.7{\scriptscriptstyle(0.5)}$}  &  {$57{\scriptscriptstyle(2.1)}$}  &  {$26{\scriptscriptstyle(1.6)}$}  &  \cellcolor{gray!25}{$81{\scriptscriptstyle(3.7)}$}   &  {$3.4{\scriptscriptstyle(1.3)}$} &  {$51{\scriptscriptstyle(6.7)}$} &  {$88{\scriptscriptstyle(14.6)}$} &  \cellcolor{gray!25}{$82{\scriptscriptstyle(3.5)}$} \\ 	
		RM \cite{motiian2017unified} & {$3.4{\scriptscriptstyle(0.4)}$ }&  {$0.66{\scriptscriptstyle(0.04)}$}  &  {$7.5{\scriptscriptstyle(0.9)}$}  &  \cellcolor{gray!25}{$74{\scriptscriptstyle(2.1)}$}  &  {$0.8{\scriptscriptstyle(0.1)}$}  &  {$82{\scriptscriptstyle(0.4)}$}  &  {$4.8{\scriptscriptstyle(0.7)}$}  &  \cellcolor{gray!25}{$84{\scriptscriptstyle(0.3)}$}  &  {$0.8{\scriptscriptstyle(0.9)}$}  &  {$52{\scriptscriptstyle(5.4)}$}  &  {$22{\scriptscriptstyle(0.6)}$}  &  \cellcolor{gray!25}{$78{\scriptscriptstyle(3.8)}$}   &  {$0.4{\scriptscriptstyle(0.5)}$} &  {$40{\scriptscriptstyle(4.7)}$} &  {$77{\scriptscriptstyle(13.8)}$} &  \cellcolor{gray!25}{$84{\scriptscriptstyle(5.3)}$} \\ 	
		Ours & {$\boldsymbol{6.4{\scriptscriptstyle(0.6)}}$ } &  {$\boldsymbol{0.54{\scriptscriptstyle(0.01)}}$}  &  {$2.7{\scriptscriptstyle(0.6)}$}  &  \cellcolor{gray!25}{$75{\scriptscriptstyle(3.3)}$}  &  {$\boldsymbol{5.3{\scriptscriptstyle(0.9)}}$}  &  {$\boldsymbol{75{\scriptscriptstyle(1.4)}}$}  &  {$7.1{\scriptscriptstyle(0.6)}$}  &  \cellcolor{gray!25}{$83{\scriptscriptstyle(0.1)}$} &  {$\boldsymbol{5.1{\scriptscriptstyle(1.2)}}$}  &  {$\boldsymbol{50{\scriptscriptstyle(4.2)}}$}  &  {$\boldsymbol{16{\scriptscriptstyle(7.2)}}$}  &  \cellcolor{gray!25}{$77{\scriptscriptstyle(4.8)}$} &  {$\boldsymbol{7.5{\scriptscriptstyle(1.2)}}$} &  {$49{\scriptscriptstyle(7.3)}$} &  {$\boldsymbol{70{\scriptscriptstyle(22.3)}}$} &  \cellcolor{gray!25}{$81{\scriptscriptstyle(1.8)}$}  \\ 
		\midrule\midrule
	\end{tabular*} 
	\captionsetup{justification=justified} 
	\caption{ \footnotesize  \textbf{Quantitative Results.} We show Mean(Std) results over multiple run. For our baselines, we consider a \textbf{Na\"ive} encoder-decoder model, learning  representations via minimizing the \textbf{MMD} criterion \cite{li2014learning} and Adversarial training \cite{NIPS2017_8cb22bdd}, termed as \textbf{CAI}. We also compare against Sub-sampling (\textbf{SS}) \cite{zhou2018statistical} that minimizes the MMD criterion separately for every age group, and  the RandMatch (\textbf{RM}) \cite{motiian2017unified} baseline that generates matching input pairs based on the Age and target label values. The SS and RM baselines discard subset of samples if a match across sites is not available. The measure $Adv$ represents the adversarial test accuracy except for the German dataset where ROC-AUC is used due to high degree of skew in the data.  } 
	\label{tab:results}
	\vspace{-0.2in}
\end{table*} 


\subsection{Pooling Brain Images across Scanners}
For our motivating application, we focus on pooling tasks for two different brain imaging datasets where the problem is to classify individuals diagnosed with Alzheimer’s disease (AD) and healthy control (CN). 

{\bf Setup. } Images are pre-processed by first normalizing and then skull-stripping using Freesurfer \cite{fischl2012freesurfer}. A linear (affine) registration is performed to register each image to MNI template space. Images are trained using 3D convolutions with ResNet \cite{he2016identity} backbone (details in the appendix \S~\ref{sec:app_nn}). Since the datasets are small, we report results over five random training-validation splits.

{\bf ADNI Dataset.} The data for this experiment has been downloaded from the Alzheimers Disease Neuroimaging Initiative (ADNI) database (adni.loni.usc.edu). We have three scanner types in the dataset, namely, GE, Siemens and Phillips. Similar to the fairness experiments, equivariance is sought relative to the covariate \textit{Age}. The values of \textit{Age} are in the range $50$-$95$ as indicated in density plot of Fig.~\ref{fig:density_comparison} (left). The \textit{Age} distribution is observed to vary across different scanners, albeit minimally, in the full dataset. In the t-SNE plot, Fig.~\ref{fig:tsne_comparison} (left), we see that the latent space has an equivariant structure. Closer inspection of the plot shows that the representations vary in the same order as that of \textit{Age}. Different colors indicate different \textbf{Age} sub-groups. Next, in Fig.~\ref{fig:tstatmap}, we present the t-statistics in the template space on the reconstructed images after pooling. Here, the t-statistics measure the association with AD/CN target labels. As seen in the figure, the voxels significantly associated with the Alzheimer's disease ($p<0.001$) are considerable in number. This result supports our goal to combine datasets to increase sample size and obtain a high power in statistical analysis. Next, in Fig.~\ref{fig:density_comparison} (right), we increase the difficulty of our problem by randomly sub-sampling for each scanner group such that the intersection of support is minimized. In such an extreme case, our method attains a better $\mathcal{M}$ metric relative to the Na\"ive method, thus justifying the applicability to situations where there is a mismatch of support across the sites. Lastly, we inspect the performance on the quantitative metrics on the entire dataset in Table~\ref{tab:results}. All metrics $\Delta_{Eq}$, $Adv$ and $\mathcal{M}$ improve relative to the baselines with a small drop in the $\mathcal{ACC}$. 

{\bf ADCP dataset.}
This experiment's data was collected as part of the NIH-sponsored Alzheimer's Disease Connectome Project (ADCP) \cite{adluru2020geodesic,hwang2018ic}. It is a two-center MRI, PET, and behavioral study of brain connectivity in AD. Study inclusion criteria for AD / MCI (Mild Cognitive Impairment) patients consisted of age $55-90$ years who retain decisional capacity at initial visit, and meet criteria for probable AD or MCI. MRI images were acquired at three sites. The three sites differ primarily in terms of the patient demographics. We inspect the quantitative results of this experiment in Tab.~\ref{tab:results} and place the qualitative results in the appendix \S~\ref{sec:app_adcp},\ref{sec:app_visualize}. The table reveals considerable improvements in all our metrics relative to the Na\"ive method. \\
{\bf Limitations.}
Currently, our 
formulation assumes that the to-be-pooled imaging datasets are roughly similar -- there 
is definitely a role for new developments 
in domain alignment to facilitate 
deployment in a broader range of applications. Secondly, larger latent space dimensions may cause compute overhead due to matrix exponential parameterization. Finally, algorithmic improvements can potentially simplify the overhead of the two-stage training. \\
\vspace{-0.2in}
\section{Conclusions}
Retrospective analysis of data pooled from previous / ongoing studies can have a sizable influence on identifying early disease processes, not otherwise possible to glean from 
analysis of small neuroimaging datasets. 
Our development based 
on recent results in equivariant representation learning 
offers a strategy to perform such analysis when covariates/nuisance attributes are not identically distributed across sites. 
Our current work is limited to a few such variables but suggests that this 
direction is promising and can 
potentially lead to more powerful 
algorithms. \\\\
{\bf Acknowledgments}
The authors are grateful to Vibhav Vineet (Microsoft Research) for discussions on the causal diagram used in the paper. Thanks to Amit Sharma (Microsoft Research) for the conversation on their MatchDG project. Special thanks to Veena Nair and Vivek Prabhakaran from UW Health for helping with the ADCP dataset. Research supported by NIH grants to UW CPCP (U54AI117924), RF1AG059312, Alzheimer's Disease Connectome Project (ADCP) U01 AG051216, and RF1AG059869, as well as NSF award CCF 1918211. Sathya Ravi was also supported by UIC-ICR start-up funds. 

{\small
\bibliographystyle{ieee_fullname}
\bibliography{egbib}
}

\appendix
\section{Appendix}

\subsection{Proofs of theoretical results }
\label{sec:app_proofs}
In this section, we will provide the proofs of Lemma~$4$ and Lemma~$5$ discussed in the main paper.
\begin{theorem*}
	\label{lemma_eqv}
	Given two latent space representations $\lat_i, \lat_j \in \mathbf{S}^{n-1}$, and the corresponding cosets $g_iH = \tau(\lat_i)$ and $g_jH = \tau(\lat_j)$, $\exists! g_{ij} = g_jg_i^{-1}\in G$ such that $\lat_j = g_{ij}\cdot \lat_i$.
\end{theorem*}
\begin{proof}
Given $g_iH = \tau(\lat_i)$ and $g_jH = \tau(\lat_j)$, we use $g_{ij} = g_ig_j^{-1} \in G$ such that, $g_jH = g_{ij}g_iH$. 

Now using the equivariance  fact~$(3)$ , we get,
\begin{align*}
	&g_jH = g_{ij}g_iH \\
	&\implies \tau(\lat_j) = g_{ij} \tau(\lat_i) \\
	&\implies \tau(\lat_j) = \tau(g_{ij}\cdot \lat_i)
\end{align*}

 Now as $\tau$ is an identification, i.e., a diffeomorphism, we get $\lat_j = g_{ij}\lat_i$. Note that  $\mathbf{S}^{n-1}$ is a Riemannian homogeneous space and the group $G$ acts transitively on $\mathbf{S}^{n-1}$, i.e., given $\mathbf{x}, \mathbf{y}\in \mathbf{S}^{n-1}$, $\exists g\in G$ such that, $\mathbf{y} = g\cdot \mathbf{x}$. Hence from $\lat_j = g_{ij}\lat_i$ and the transitivity property we can conclude that $g_{ij}$ is unique.
\end{proof}

\begin{theorem*}
	\label{arb_equiv}
	For a $\tau: \mathcal{L}\rightarrow G/H$ as defined above, and a mapping $b:\mathcal{L}\rightarrow \mathcal{Z}$, the function $\Phi: \mathcal{L}\rightarrow \mathcal{Z}$ defined by
	\begin{align}
		\label{eq_arb}
		\Phi(\lat) = \tau(\lat)\cdot b\left(\tau(\lat)^{-1}\cdot \lat\right) 
	\end{align}
	is $G$-equivariant, i.e., $\Phi(g\cdot \lat) = g\Phi(\lat)$.
\end{theorem*}
\begin{proof}
Let $\lat\in \mathcal{L}$. Consider the $\Phi$ mapping of $g \cdot \lat$, that is  $\Phi(g\cdot \lat) = \tau(g\cdot \lat)\cdot b\left(\tau(g\cdot \lat)^{-1}\cdot \lat\right)$.

 Using the fact~$(3)$ from the main paper, we have $\tau(g\cdot \lat) = g\tau(\lat)$ and $\tau(g\cdot \lat)^{-1} = \tau(\lat)^{-1}g^{-1}$. Substituting these in $\Phi(g\cdot \lat)$, we get 
 \begin{align*}
 	\Phi(g\cdot \lat) &= g\tau( \lat)\cdot b\left(\tau(\lat)^{-1}g^{-1}g\cdot \lat\right)\\
 	&=g\tau(\lat)b\left(\tau(\lat)^{-1}\cdot \lat\right)\\
 	\text{Thus, } \Phi(g\cdot \lat) &= g\Phi(\lat)
 \end{align*}
\end{proof}


\begin{figure*}[!t]
	\begin{subfigure}[b]{0.45\linewidth}
		\includegraphics[width=\linewidth]{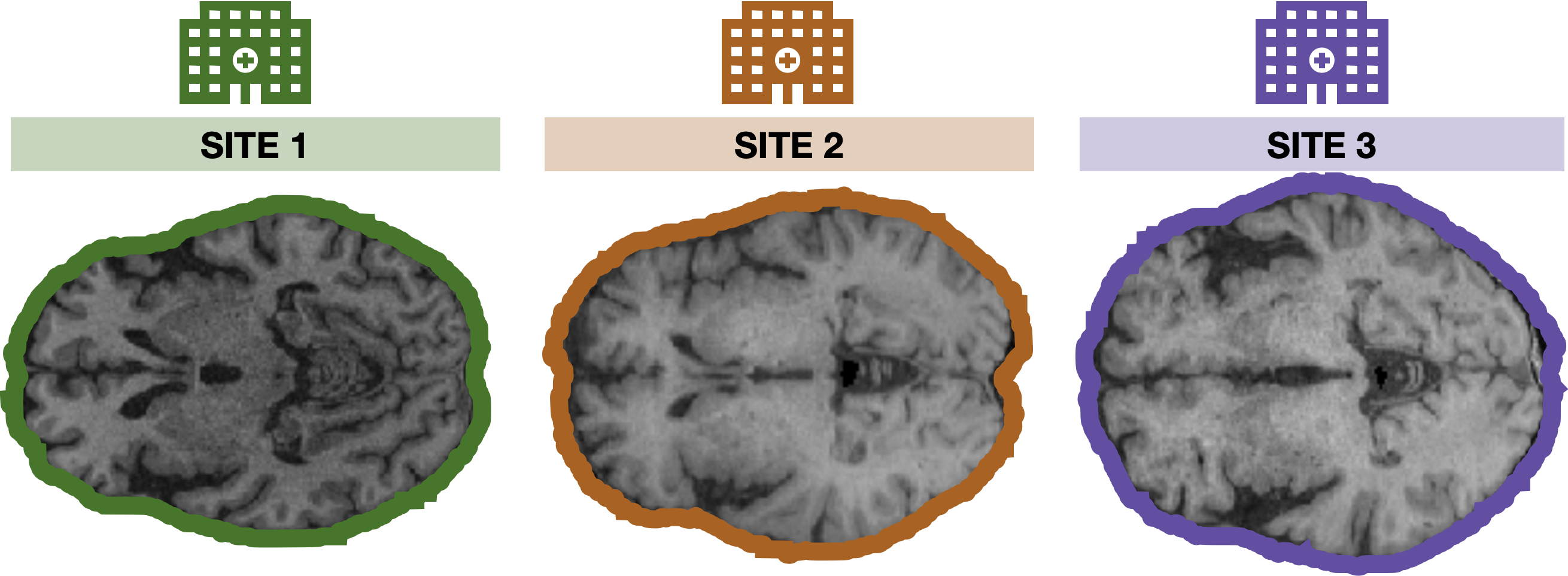}
		\caption{Variation due to scanner for particular age group.}
		\label{fig:adcp_confounds_site}
	\end{subfigure} \hspace{4mm}%
	\begin{subfigure}[b]{0.45\linewidth}
		\includegraphics[width=\linewidth]{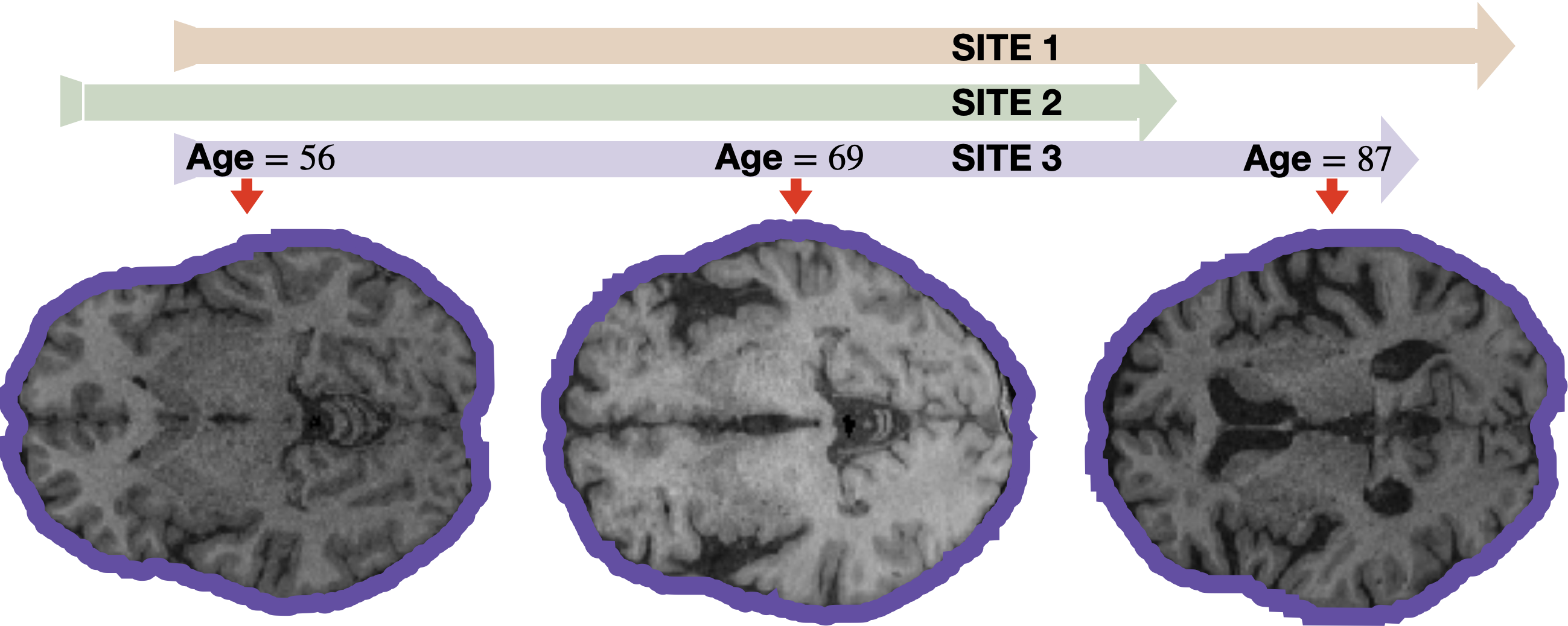}
		\caption{Variation due to covariates (age) in scanner $3$.}
		\label{fig:adcp_confounds_covariate}
	\end{subfigure}%
	\centering
	\caption{\footnotesize \label{fig:adcp_variations} {\bf Sample Images from ADCP dataset. }  \textbf{(a)} MRI images on control subjects from the ADCP dataset for different \textbf{sites} in the age group $70$-$80$.  \textbf{(b)} Images obtained from Site $3$ for three extreme \textbf{age groups}. The gantt chart on top of the image indicates the respective age range in the other sites.     }
	\vspace{-0.2in}
\end{figure*}

\subsection{Details on Evaluation Metrics}
\label{sec:app_evaluation}
Recall from Section~$4$ of the paper, our discussion on three metrics -- $\boldsymbol{\Delta_{Eq}}$, $\boldsymbol{Adv}$ and $\boldsymbol{\mathcal{M}}$. While $\boldsymbol{\Delta_{Eq}}$ and $\boldsymbol{\mathcal{M}}$ are variants of distance measure on the latent space, $\boldsymbol{Adv}$ assesses the ability to predict the nuisance attributes from the latent representation (and is therefore probabilistic in nature). Observe that $\boldsymbol{\Delta_{Eq}}$ and $\boldsymbol{\mathcal{M}}$ are (euclidean) distance measures and could be very different depending on the normalization of the vectors. For our purposes of evaluating these latent vectors/features in downstream tasks, 
we perform a simple feature normalization in order to obtain $0-1$   latent vectors given by,
\begin{align}
 	\tilde{z_i} = \frac{z_i  - \min(z_i)}{\max(z_i) - \min(z_i)}.\label{eq:z_norm}
\end{align}
Our feature normalization is composed of two steps: (i) centering -- the numerator in \eqref{eq:z_norm} ensures that the mean of $z$ (along its coordinates) is $0$; and (ii) scale -- the denominator projects the features $z$ on the sphere at origin with radius $\|z_i\|_{\infty}^{\geq}  = \max(z_i) - \min(z_i) \geq 0$. Note that our scaling step can be thought of as the usual projection in a special case: when $z_i$ is guaranteed to be nonnegative (for example, when $z_i$ represent activations), then $\|z_i\|_{\infty}^{\geq}$ simply corresponds  to a lower bound of the usual infinity norm, $\|z\|_{\infty}$ (hence projection on a scaled $\ell_{\infty}$ ball). We adopt this normalization only to compute $\boldsymbol{\Delta_{Eq}}$ and $\boldsymbol{\mathcal{M}}$ measures, and not for model training.

For computing the $\boldsymbol{Adv}$ measure, we follow \cite{NIPS2017_8cb22bdd} to train an adversarial neural network predicting the nuisance attributes. We use a three-layered fully connected network with batch normalization and train for $150$ epochs. \cite{NEURIPS2018_415185ea}  uses similar architecture for the adversaries with different hidden layers of $0, 1, 2, 3$. We found that a three-layer adversary is powerful enough to predict the nuisance attributes and hence we use it to report the  $\boldsymbol{Adv}$ measure.

\subsection{Understanding ADNI dataset}
\label{sec:app_adni}
\noindent\textbf{Dataset.} The data was downloaded from the Alzheimers Disease Neuroimaging Initiative (ADNI) database (adni.loni.usc.edu). The ADNI was launched in 2003 as a public-private partnership, led by Principal Investigator Michael W. Weiner, MD. ADNI was set up with an objective to measure the progression of mild cognitive impairment (MCI) and early Alzheimers disease (AD) using serial magnetic resonance imaging (MRI), positron emission tomography (PET), other biological markers. We have three imaging protocol (scanner) types in the dataset, namely, GE, Siemens and Phillips. The count of samples AD/CN in each of these imaging protocols are provided in Table \ref{tab:adni}. An example illustration (borrowed from \cite{aisen2017path}) of using different scanner on the images is shown in Figure \ref{fig:adni_img}.

\begin{table}[!b]
	\caption{Sample counts for ADNI dataset}
	\label{tab:adni}
	\centering
	\begin{tabular}{lll}
		\toprule
		Imaging Protocol     & AD     & CN \\
		\midrule
		Manufacturer=GE Medical Systems 	 & $44$  & $78$     \\
		Manufacturer=Philips Medical Systems    & $32$ & $50$      \\
		Manufacturer=Siemens    & $83$  & $162$  \\
		\bottomrule
	\end{tabular}
\end{table}

\begin{figure}
	\centering
	\begin{subfigure}[t]{0.4\columnwidth}
		\centering
		\includegraphics[width=0.9\linewidth]{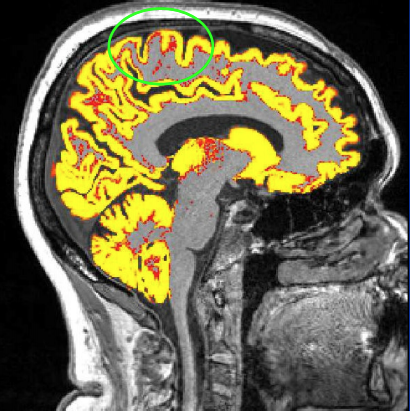}
		\caption{GE}
	\end{subfigure} \hspace{0.1in} %
	\begin{subfigure}[t]{0.4\columnwidth}
		\centering
		\includegraphics[width=0.9\linewidth]{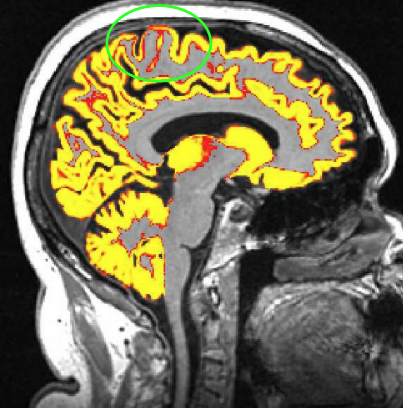}
		\caption{Siemens}
	\end{subfigure}
	\caption{\footnotesize {\bf Scanner effects on images.} Two  imaging protocols are shown: (a) Siemens, (b) GE. The yellow region is the cortical ribbon segmentation, and the green circle shows that the imaging protocol from different manufacturers have an effect on the scan. Image borrowed from \cite{aisen2017path}. }
	\label{fig:adni_img}
	\vspace{-0.2in}
\end{figure}
\noindent\textbf{Preprocessing.} All images were first normalized and skull-stripped using Freesurfer \cite{fischl2012freesurfer}. A linear (affine) registration was performed to register each image to MNI template space.

\subsection{Understanding ADCP dataset}
\label{sec:app_adcp}
\noindent\textbf{Participants.}  The data for ADCP was collected through an NIH-sponsored Alzheimer's Disease Connectome Project (ADCP) U01 AG051216. The study inclusion criteria for AD (Alzheimer's disease) / MCI (Mild Cognitive Impairment) patients consisted of age between  $55$-$90$ years, willing and able to undergo all procedures, retains decisional capacity at initial visit, meets criteria for probable AD or meets criteria for MCI. 

\noindent\textbf{Scanners.}  MRI images were acquired at three distinct sites on GE scanners. T1-weighted structural images were acquired using a 3D gradient-echo pulse sequence (repetition time (TR) = $604$ ms, echo time (TE) = $2.516$ ms, inversion time = $1060$ ms, flip angle = $8^\text{o}$, field of view (FOV) = $25.6$ cm, $0.8$ mm isotropic). T2-weighted structural images were acquired using a 3D fast spin-echo sequence (TR = $2500$ ms, TE = $94.398$ ms, flip angle = $90^\text{o}$, FOV = $25.6$ cm, $0.8$ mm isotropic). 

\noindent\textbf{Preprocessing.} The Human Connectome Project (HCP) minimal preprocessing pipeline version $3.4.0$ \cite{glasser2013minimal} was followed for data processing. This pipeline is based on FMRIB Software Library \cite{jenkinson2012fsl}. Next, the T1w and T2w images are aligned, a B1 (bias field) correction is performed, and the subject’s image in native structural volume space is registered to MNI space using FSL’s FNIRT \cite{andersson2007non}. Only T1w images in the MNI space were used for further analysis and experiments. 

\noindent\textbf{Data Statistics.} We plot the distributions of several attributes in this dataset conditioned on the site. In Figure~\ref{fig:demo1}, we show that the values of age and cognitive scores differ across the three sites in this dataset. Cognitive scores are computed based on an test assigned to the patients. Higher scores indicate higher cognitive operation in the patient. Table~\ref{tab:demo2} shows the sample counts for target variable of prediction AD (Alzheimer's disease) and Control group.

\begin{figure*}[!t]
	\begin{subfigure}[b]{0.4\linewidth}
		\includegraphics[width=\linewidth]{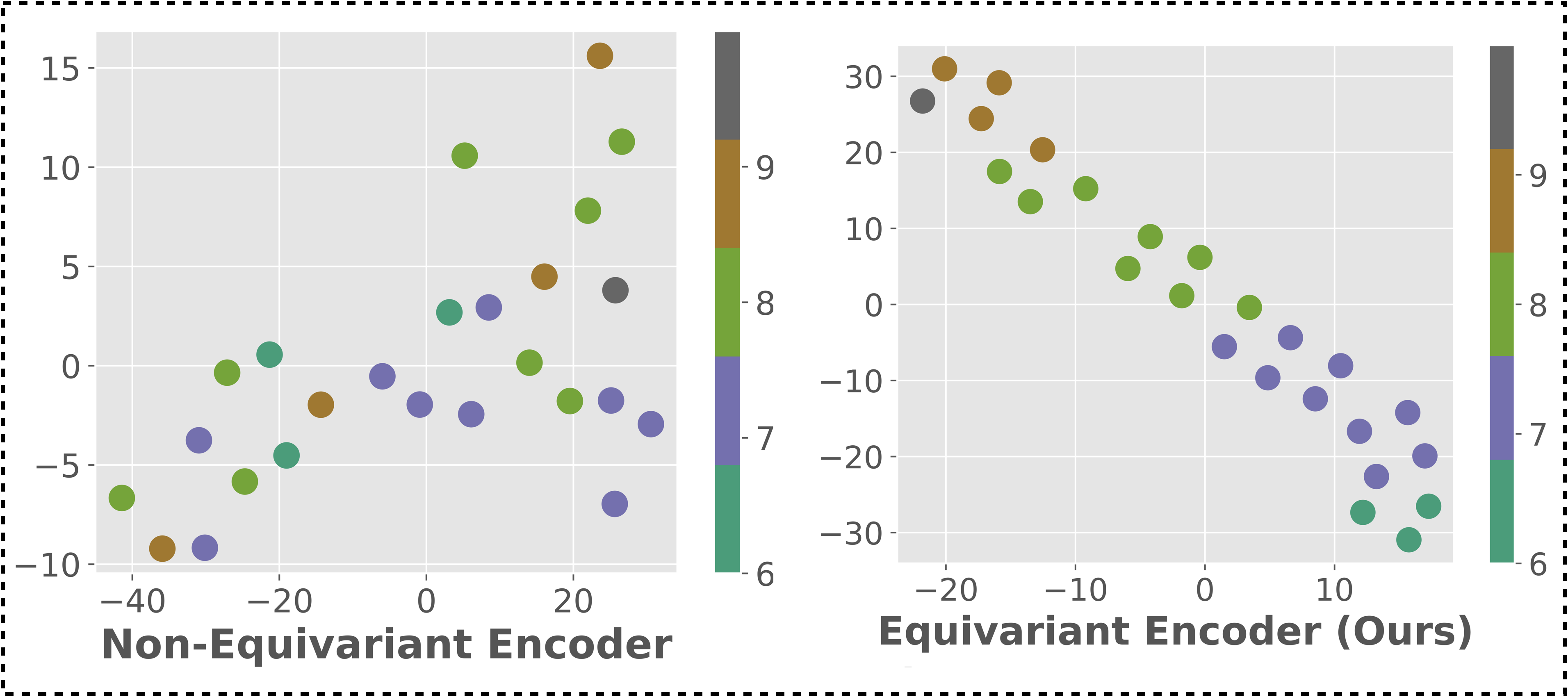}
		\caption{ADCP Dataset}
	\end{subfigure} \hspace{1cm}
	\begin{subfigure}[b]{0.4\linewidth}
		\includegraphics[width=\linewidth]{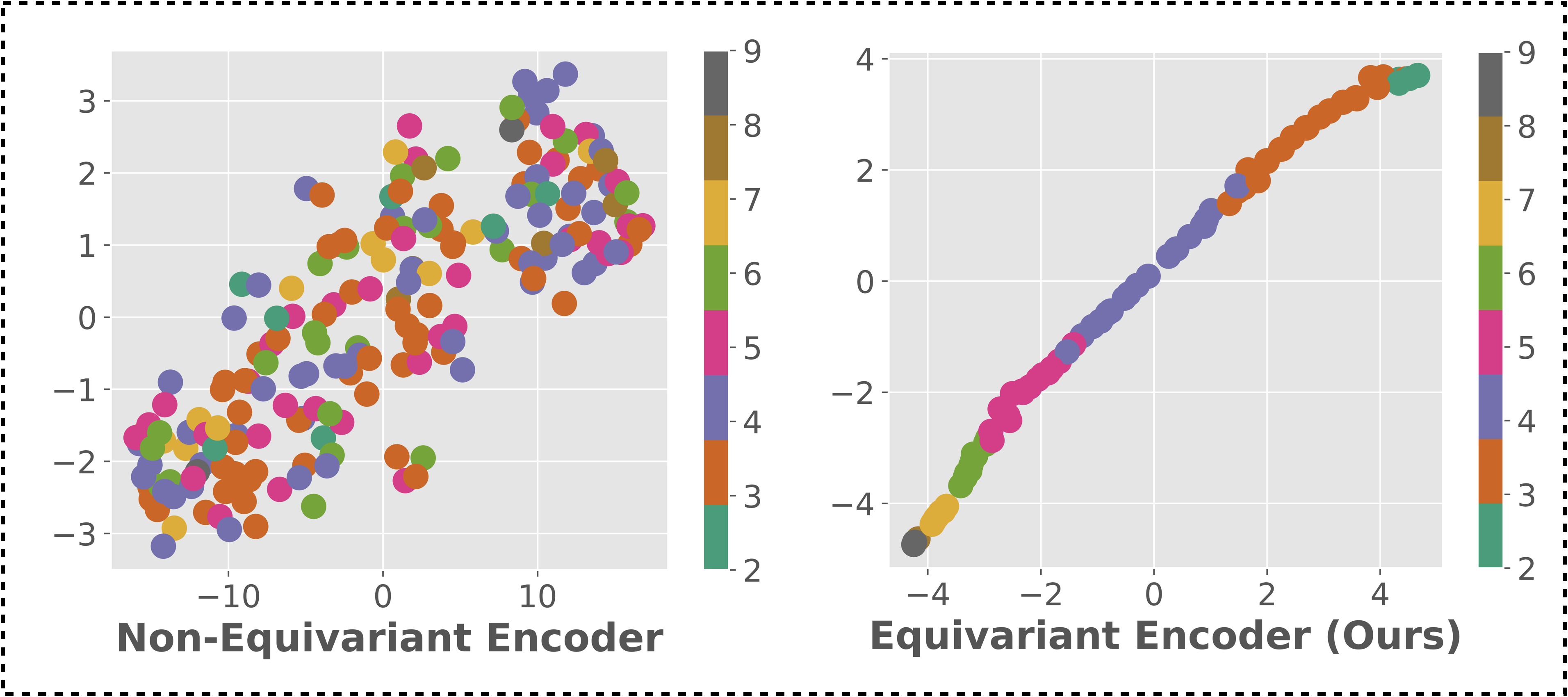}
		\caption{German Dataset}
	\end{subfigure}%
	\centering
	\caption{\footnotesize {\bf t-SNE plots of latent representations of $\tau(\lat)$ }. For both ADCP \textbf{(left)} and German \textbf{(right)}, the  the latent vectors of the equivariant encoder are evenly distributed with respect to the age covariate value. The non-equivariant space is generated from the na\"ive pooling model. Different colors denote the discretized set of \textbf{age} covariate value present in the data.}
	\label{fig:supp_tsne_comparison}
\end{figure*}

\begin{figure}[!t]
	\begin{subfigure}[b]{0.49\columnwidth}
		\includegraphics[width=\linewidth]{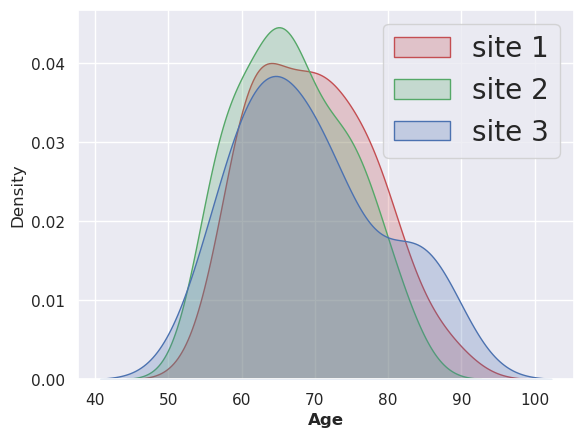}
		\caption{Age}
	\end{subfigure} 
	\begin{subfigure}[b]{0.49\columnwidth}
		\includegraphics[width=\linewidth]{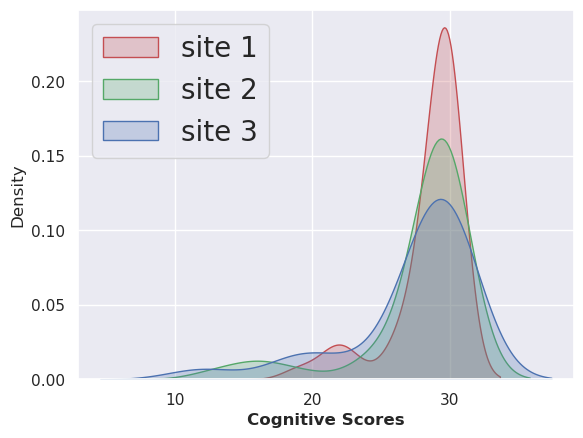}
		\caption{Cognitive Score}
	\end{subfigure}%
	\centering
	\caption{ \footnotesize {\bf Distribution of attributes in the ADCP dataset. } On the \textbf{left} we observe the distribution of age for the three different sites present in the ADCP dataset. On the \textbf{right}, we see the distribution of the cognitive scores. The cognitive scores are computed based on a test that assesses executive function. Higher scores indicate higher level of cognitive flexibility. Both age and cognitive scores are observed to vary across the sites. }
	\label{fig:demo1}
\end{figure}

\begin{table}[!b]
	\caption{Sample counts for ADCP dataset}
	\label{tab:demo2}
	\centering
	\begin{tabular}{llllll}
		\toprule
		& AD     & Control & & Female     & Male \\
		\midrule
		site $1$ 	 & $10$  & $39$ & & $29$  & $20$       \\
		site $2$     & $10$ & $33$   & & $30$  & $13$    \\
		site $3$     & $5$  & $19$  & & $14$  & $10$ \\
		\bottomrule
	\end{tabular}
\end{table}

\subsection{Visualizing the latent space}
\label{sec:app_visualize}
In the paper Figure~$4$, we have seen the latent space $\tau(\lat)$ for the samples in the ADNI and the Adult datasets. Here, we will see similar qualitative results for the German and the ADCP dataset in Figure~\ref{fig:supp_tsne_comparison} of the supplement. In the plots, the latent representations for a non-equivariant encoder are stretched thoughout the latent space. In contrast, the representations of an equivariant encoder, for a discretized value of \textbf{Age}, are localized to specific regions. Further, these representations have a monotonic behaviour with respect to the values of \textbf{Age}. 
\noindent\begin{minipage}{.45\textwidth}
	\begin{lstlisting}[caption=Residual Block,frame=tlrb]{Name}
		BatchNorm3d
		Swish
		Conv3d
		BatchNorm3d
		Swish
		Conv3d
	\end{lstlisting}
\end{minipage}\hfill
\begin{minipage}{.45\textwidth}
	\begin{lstlisting}[caption=Fully Connected Block,frame=tlrb]{Name}
		AdaptiveAvgPool3d
		Flatten
		Dropout
		Linear
		BatchNorm1d
		Swish
		Dropout
		Linear
	\end{lstlisting}
\end{minipage}
\subsection{Hyper-parameters and NN Architectures}
\label{sec:app_nn}
For tabular datasets such as German and Adult, our encoders and decoders comprise of fully connected networks and a hidden layer of $64$ nodes. The dimension of the quotient latent space $\tau(\lat_i)$ is $30$. Adam is used as a default optimizer and the learning rate is adjusted based on the validation set. 

Imaging datasets like ADNI and ADCP require 3D convolutions and a ResNet architecture as the backbone. The last layer is used to describe the quotient space $\tau(\lat_i)$. We present the residual and the fully connected block below. Detailed architectures can be viewed in the code. 

\subsection{Scaling factors}
\label{sec:app_scale}
Recall from the Algorithm~$1$ of the main paper that our loss function for each stage comprises of reconstruction and prediction losses in addition to the objectives concerning equivariance and invariance. These multi-objective loss functions require scaling factors that upweight one objective over the other. These scaling factors group up as hyper-parameters for the Algorithm. In our experiments, it was observed that the results were robust to a range of scaling factor choices. For the results reported in Table~$1$ of the paper, they were identified through cross-validation. Here we provide an example for the scaling factors used for the Adult dataset, please refer to the bash scripts available in the code for the scaling factors of other datasets.
\begin{itemize}
	\item \textbf{Stage one: Equivariance to Covariates}
	\begin{itemize}
		\item Equivariance Loss $L_{\text{stage} 1}$ \\ Scaling Factor $:1.0$
		\item Reconstruction Loss $\sum_i \|X_i - \mathfrak{D}(\mathfrak{E}(X_i)) \|$\\ Scaling Factor $:0.02$
	\end{itemize}
	\item \textbf{Stage two: Invariance to Site}
		\begin{itemize}
		\item Invariance Loss $\mathcal{MMD}$\\ Scaling Factor $: 0.1$
		\item Prediction Loss $\|Y-h(\Phi(\lat))\|^2$\\ Scaling Factor $: 1.0$
		\item Reconstruction Loss $\|\lat - \Psi(\Phi(\lat))\|^2$\\ Scaling Factor $: 0.1$
	\end{itemize}
\end{itemize}
We refer the reader to Algorithm~$1$ and Section~$3$ of the main paper for the details on the notations used above.

\end{document}